\documentclass{article}

\usepackage[accepted]{aistats2021}
\usepackage{enumitem}
\usepackage[utf8]{inputenc} 
\usepackage{url}            
\usepackage{booktabs}       
\usepackage{amsfonts}       
\usepackage{numprint}
\usepackage{amsmath}
\usepackage{bbm}
\usepackage{bm,amsbsy} 
\usepackage{amsfonts,amsmath,amssymb,amsthm,mathtools}
\usepackage{graphicx}
\usepackage{standalone}
\usepackage{fancyhdr}

\usepackage{multirow}
\usepackage{fancyhdr}
\usepackage{color}
\usepackage{algorithm}
\usepackage{algorithmic}

\usepackage{eso-pic} 
\usepackage{forloop}
\usepackage{times}
\usepackage{wrapfig}
\usepackage{siunitx}
\usepackage{hyperref}

\newcommand{\punt}[1]{}

\usepackage{amsthm}
\usepackage{xcolor}
\usepackage{hyperref}

\newcommand{\trp}{{^\top}} 

\renewcommand{\eqref}[1]{eq.~\ref{eq:#1}}
\newcommand{\Nrm}{\mathcal{N}}

\newcommand{\figref}[1]{Fig.~\ref{fig:#1}}  
\newcommand{\secref}[1]{Sec.~\ref{sec:#1}}
\newcommand{\tabref}[1]{Table ~\ref{tab:#1}}  
\newcommand{\algoref}[1]{Algorithm~\ref{algo:#1}}  
\newcommand{\propref}[1]{Prop.~\ref{prop:#1}}

\newcommand{\lemmaref}[1]{Lemma.~\ref{lemma:#1}}

\newcommand{\suppsecref}[1]{Supplementary Sec.~\ref{supp:#1}}  

\newcommand{\vx}{\mathbf{x}}

\newcommand{\Dat}{\mathcal{D}}

\newcommand{\vf}{\mathbf{f}}

\newcommand{\vphi}{\mathbf{\ensuremath{\bm{\phi}}}}

\newcommand{\vu}{\mathbf{u}}

\newcommand{\vm}{\mathbf{m}}
\newcommand{\vz}{\mathbf{z}}

\newcommand{\vmu}{\mathbf{\ensuremath{\bm{\mu}}}}
\newcommand{\vtheta}{\mathbf{\ensuremath{\bm{\theta}}}}

\newcommand{\mI}{\mathbf{I}}

\newcommand{\vn}{\mathbf{n}}

\newcommand{\Em}{\mathbb{E}}

\newcommand{\vy}{\mathbf{y}}
\newcommand{\vh}{\mathbf{h}}

\newtheorem{thm}{Theorem}[section]
\newtheorem{lem}{Lemma}[section]

\newtheorem{prop}[thm]{Proposition}
\newtheorem{remark}{Remark}

\def\argmin{\mathop{\rm arg\,min}}

\def\Pr{\ensuremath{\text{Pr}}}

\begin{document}


\twocolumn[
\aistatstitle{DP-MERF: Differentially Private Mean Embeddings with Random Features 
for Practical Privacy-Preserving Data Generation}
\aistatsauthor{
Frederik Harder${}^{* 1, 2}$
\And 
Kamil Adamczewski${}^{* 1, 3}$
\And 
Mijung Park${}^{1, 2}$
}
\aistatsaddress{    
${}^1$ Max Planck Institute for Intelligent Systems, Tübingen, Germany \\
${}^2$ Department of Computer Science, University of Tübingen, Tübingen, Germany \\
${}^3$ D-ITET, ETH Zurich, Switzerland\\
 \texttt{\{fharder|kadamczewski|mpark\}@tue.mpg.de} 
}
]




\begin{abstract}
    We propose a differentially private data generation paradigm using random feature representations of kernel mean embeddings when comparing the distribution of true data with that of synthetic data. We exploit the random feature representations for two important benefits.
First, we require a minimal privacy cost for training deep generative models. This is because unlike kernel-based distance metrics that require computing the kernel matrix on all pairs of true and synthetic data points, we can detach the data-dependent term from the term solely dependent on synthetic data. Hence, we need to perturb the data-dependent term only once and then use it repeatedly during the generator training.
Second, we can obtain an analytic sensitivity of the kernel mean embedding as the random features are norm bounded by construction. This removes the necessity of hyper-parameter search for a clipping norm to handle the unknown sensitivity of a generator network. We provide several variants of our algorithm, differentially-private mean embeddings with random features (DP-MERF) to jointly generate labels and input features for datasets such as heterogeneous tabular data and image data. Our algorithm achieves drastically better privacy-utility trade-offs than existing methods when tested on several datasets.

\end{abstract}




\section{Introduction}
\label{sec:Introduction}

Differential privacy (DP) is a gold standard privacy notion that is widely used in many applications in machine learning. However, due to its composability,  
%
%
every access to data reduces the privacy guarantee, which limits the number of times one can query sensitive data before a desired privacy level is exceeded.
%
Differentially private data generation solves this problem of limited access by creating a synthetic dataset that is \emph{similar} to the true dataset using DP mechanisms. This process also comes at a privacy cost, but afterwards, the synthetic dataset can be used in place of the true one for unlimited time without further loss of privacy. 

Classical approaches to differentially private data generation typically assume a certain class of pre-specified queries. These DP algorithms produce a privacy-preserving synthetic database that is similar to the privacy-sensitive original data for that \textit{fixed query class}
\cite{Mohammed:2011:DPD:2020408.2020487, 10.1007/978-3-642-15546-8_11, NIPS2012_4548, 7911185}.
However, specifying a query class in advance, significantly limits the flexibility of the synthetic data, if data analysts hope to perform other machine learning tasks.

To overcome this inflexibility, recent papers on DP data generation have utilized deep generative modelling.
The majority of these approaches is based on the generative adversarial networks (GAN) \cite{NIPS2014_5423} framework, where a discriminator and a generator play a  min-max form of game to optimize a given distance metric between the true and  synthetic data distributions. 
Most approaches have used either the \textit{Jensen-Shannon divergence} \cite{DP_EM, DP_CGAN, PATE_GAN}, or the Wasserstein distance \cite{DPGAN, DBLP:conf/sec/FrigerioOGD19}. For more details on different divergence metrics, see \suppsecref{background_distances}.

Another popular choice of distance metric for generative modelling is Maximum Mean Discrepancy (MMD). MMD can compare two probability measures in terms of all possible moments. Therefore, there is no information loss due to a selection of a certain set of moments. The MMD estimator is in closed form (\eqref{MMD_full}) and easy to compute by the pair-wise evaluations of a kernel function using the points drawn from the true and the generated data distributions.
 
In this work, we propose to use a particular form of MMD via \textit{random Fourier feature} representations \cite{rahimi2008random} of kernel mean embeddings for DP data generation. While MMD can be used within a GAN framework as well (see e.g. \cite{mmd_gan}) we choose a much simpler method, which is particularly suited for training with DP constraints. 
%

In the objective we use (\eqref{MMD_rf}), the mean embedding of the true data distribution (data-dependent) is separate from the embedding of the synthetic data distribution (data-independent).
Hence, only the data-dependent term requires privatization. 
Random features provide an analytic sensitivity of the mean embedding, which allows us to release a DP version of this embedding through a DP mechanism as we explain below. With the privatized data embedding and the synthetic data embedding, our objective no longer directly accesses the data and can be optimized freely to train  a data generator.
Our contributions are summarized below. 
%

 \textbf{(1) We provide a simple algorithm for DP data generation, which improves on existing methods both in privacy and utility.}
\begin{itemize}[topsep=0pt]
\setlength{\itemsep}{1pt}
\item \textit{Simple to optimize:} 
Since the objective of the optimization contains only a specific private release of data, there are no privacy induced constraints on model choice and optimization method due to privacy. In contrast, methods with private releases as part of the training loop are generally constrained in the number of iterations. As a specific example, DP-SGD requires well-defined sample-wise gradients, which prohibits the use of batch-normalization.
Further, increasing the number of trained weights raises the sensitivity of DP-SGD \cite{DP_SGD} and with it the required strength of gradient perturbation, making large networks infeasible.
Our method also avoids the cumbersome min-max optimization present in GAN based approaches and requires only a minimal number of hyperparameters\footnote{Hyperparameters in our method are the number of random features, a kernel parameter, and the learning rate.}. 
\item \textit{Strong privacy:} 
Computing the \textit{sensitivity} in our method is \textit{analytically tractable} due to its norm-boundedness of random features. In fact, the norm of random features we use is bounded by 1 by construction. The resulting sensitivity is on the order of 1 over the number of training data points. Consequently, a moderate size of training data can significantly reduce the sensitivity. 
By requiring only a single DP-release with such a low sensitivity, our method can provide strong DP guarantee more easily than methods which access the data on each training iteration.
\item \textit{High utility:}
We show in our experiments that our method releases private data with higher utility for downstream tasks than comparison methods.  
This contrast is particularly stark on MNIST, where our model at a strong privacy guarantee of $(0.2, 10^{-5})$-DP outperforms all GAN-based comparison methods, even though they are trained with much weaker privacy of at most $(9.6, 10^{-5})$-DP.

\item \textit{Theoretical study:} 
We provide an error bound on the objective to theoretically quantify the effect of noise added for privacy to the random feature representation of MMD objective. 
This bound provides an informative way to select the random feature dimension, given a dataset size and a desired privacy level.  
\end{itemize}
    
\textbf{(2) Our algorithm accommodates several needs in privacy-preserving data generation.} 
\begin{itemize}[topsep=0pt]
\item \textit{Generating input and output pairs jointly}:
We treat both input and output to be privacy-sensitive. This is different from the conditional-GAN type of methods, where the class distribution is treated as non-sensitive, which increases the risk of successful membership inference, particularly in imbalanced datasets where some classes contain only a small number of samples. %
\item \textit{Generating imbalanced and heterogeneous tabular data}: 
Real world datasets may exhibit large variation in data types and class sizes.
By addressing both of these issues, we ensure that our algorithm is applicable to a wide variety of datasets. 
\end{itemize}

We start by describing relevant background information in \secref{background} before introducing our method in \secref{Methods} and \secref{dpmerf_plus}, followed by an overview of related work in \secref{related_work} and experiments in \secref{experiments}. 

\section{Background} \label{sec:background}

In the following, we describe the kernel mean embeddings with random features and differential privacy, which our model will use in \secref{Methods}. 

\subsection{Maximum Mean Discrepancy}

Given a positive definite kernel
$k\colon\mathcal{X}\times\mathcal{X}$, the MMD between two distributions $P,Q$ is defined as \cite{Gretton2012}
\begin{align}\label{eq:pop_mmd}
\mathrm{MMD}^2(P,Q)&=
 \mathbb{E}_{x,x'\sim P}k(x,x')+\mathbb{E}_{y,y'\sim Q}k(y,y') \nonumber \\
 & \qquad -2\mathbb{E}_{x\sim P}\mathbb{E}_{y\sim Q}k(x,y).
\end{align}
According to the Moore--Aronszajn theorem, there exists a unique Hilbert
space $\mathcal{H}$ on which $k$ defines an inner product. Hence, we can find a feature map $\phi\colon\mathcal{X}\to\mathcal{H}$
such that $k(x,y)=\left\langle \phi(x),\phi(y)\right\rangle _{\mathcal{H}}$, 
where $\left\langle \cdot,\cdot\right\rangle _{\mathcal{H}}=\left\langle \cdot,\cdot\right\rangle $
denotes the inner product on $\mathcal{H}$. Using this fact, we can rewrite the MMD in \eqref{pop_mmd} as \cite{Gretton2012}
\begin{align*}
\mathrm{MMD}(P,Q) & =\big\|\mathbb{E}_{x\sim P}[\phi(x)]-\mathbb{E}_{y\sim Q}[\phi(y)]\big\|_{\mathcal{H}},
\end{align*}
where $\mathbb{E}_{x\sim P}[\phi(x)]\in\mathcal{H}$ is known as the
(kernel) mean embedding of $P$, and exists if $\mathbb{E}_{x\sim P}\sqrt{k(x,x)}<\infty$
\cite{Smola2007}. 
The MMD can be interpreted as the distance
between the mean embeddings of the two distributions. 
If $k$ is a
characteristic kernel \cite{Sriperumbudur2011}, then $P\mapsto\mathbb{E}_{x\sim P}[\phi(x)]$
is injective, and MMD forms a metric, implying that $\mathrm{MMD}(P,Q)=0$, if and only if $P=Q$. 

Given the samples drawn from two probability distributions: $X_{m}=\{x_{i}\}_{i=1}^{m} \sim P$ and $X'_{n}=\{x'_{i}\}_{i=1}^{n} \sim Q$, we can estimate\footnote{Note that this particular MMD estimator is biased.} the MMD by sample averages \cite{Gretton2012}:
\begin{align}\label{eq:MMD_full}
\widehat{\mathrm{MMD}}^2(X_{m},X'_{n}) &= \tfrac{1}{m^2}\sum_{i,j=1}^{m}k(x_{i},x_{j}) +\tfrac{1}{n^2}\sum_{i,j=1}^{n}k(x'_{i},x'_{j}) 
\nonumber \\
& \qquad 
 -\tfrac{2}{mn}\sum_{i=1}^{m}\sum_{j=1}^{n}k(x_{i},x'_{j}).
\end{align}
However, the total computational cost of $
\widehat{\mathrm{MMD}}(X_{m},X'_{n}) $ is $O(mn)$, which is prohibitive for large-scale datasets.

\subsection{Random feature mean embeddings}\label{subsec:RFME}

%
%
%
%
A fast linear-time MMD estimator can be achieved by considering an approximation to the kernel function $k(x,x')$ with an inner product of finite dimensional feature vectors, i.e.,  $k(x,x')\approx \hat{\phi}(x)^{\top}\hat{\phi}(x')$ where $\hat{\phi}(x)\in\mathbb{R}^{D}$ and $D$ is the number of
features. 
The resulting approximation of the MMD estimator given in \eqref{MMD_full} can be computed in $O(m+n)$, i.e., linear in the sample size:
\begin{align}\label{eq:MMD_rf}
\widehat{\mathrm{MMD}}_{rf}^{2}(P,Q)=\bigg\|\tfrac{1}{m}\sum_{i=1}^{m}\hat{\phi}(x
_i)-\tfrac{1}{n}\sum_{i=1}^{n}\hat{\phi}(x'_i)\bigg\|_{2}^{2},  
\end{align}
%
%
One popular approach to obtaining such $\hat{\phi}(\cdot)$
is based on random Fourier features \cite{rahimi2008random} which can be applied to any translation invariant kernel, i.e., $k(x,x')=\tilde{k}(x-x')$ for some function $\tilde{k}$. According to Bochner's theorem \cite{Rudin2013}, $\tilde{k}$ can be written as
$
\tilde{k}(x-x') =\int e^{i\omega^{\top}(x-x')}\,\mathrm{d}\Lambda(\omega)
 =\mathbb{E}_{\omega\sim\Lambda}\cos(\omega^{\top}(x-x')),
$
where $i=\sqrt{-1}$ and due to positive-definiteness of $\tilde k$, its Fourier transform
$\Lambda$ is nonnegative and can be treated as a probability measure. By drawing
random frequencies $\{\omega_{i}\}_{i=1}^{D}\sim\Lambda$, where $\Lambda$ depends on the kernel,
(e.g., a Gaussian kernel $k$ corresponds to
normal distribution $\Lambda$), 
$\tilde{k}(x-x')$ can be
approximated with a Monte Carlo average. 
The vector of random Fourier features is given by 
\begin{align}\label{eq:RF}
    \hat{\vphi}(x)=(\hat{\phi}_{1}(x),\ldots,\hat{\phi}_{D}(x))^{\top}
\end{align} where each coordinate is defined by 
%
\begin{align}
    \hat\phi_{j}(x) & = \sqrt{2/D}\;\cos(\omega_j\trp x), \nonumber \\
      \hat{\phi}_{j+D/2}(x)& =\sqrt{2/D}\sin(\omega_{j}^{\top}x), \nonumber 
\end{align} for $j=1, \cdots, D/2$.
%
%
%
The approximation error due to these random features was studied in  \cite{Dougal_UAI}. 

\subsection{Differential privacy}\label{subsec:DP}

Given privacy parameters $\epsilon \geq 0$ and $\delta \geq 0$, a mechanism $\mathcal{M}$ is  ($\epsilon$, $\delta$)-DP if and only if for all possible sets of mechanism outputs $S$ and all neighbouring datasets $\Dat$, $\Dat'$ differing by a single entry, the following equation holds:
\begin{align}
\Pr[\mathcal{M}(\Dat) \in S] \leq e^\epsilon \cdot \Pr[\mathcal{M}(\Dat') \in S] + \delta
\end{align}
%
%
%
A DP mechanism guarantees a limit on the amount of information revealed about any one individual in the dataset. Typically this guarantee is achieved by adding randomness to the algorithms' output. Let a function $h: \Dat \mapsto \mathbb{R}^p$, which is computed on sensitive data $\Dat$, output a $p$-dimensional vector. We can add noise to $h$ for privacy, where the level of noise is calibrated to the {\it{global sensitivity}}
\cite{dwork2006our}, $\Delta_h$, defined by the maximum difference in terms of $L_2$-norm $||h(\Dat)-h(\Dat') ||_2$, for neighbouring $\Dat$ and $\Dat'$ (i.e. $\Dat$ and $\Dat'$ have one sample difference by replacement). 
The \textit{Gaussian mechanism} that we will use in this paper outputs $\widetilde{h}(\Dat) = h(\Dat) + \Nrm(0, \sigma^2 \Delta_h^2\mathbf{I}_p)$. 
The perturbed function $\widetilde{h}(\Dat) $ is $(\epsilon, \delta)$-DP, where $\sigma$ is a function of $\epsilon$ and $\delta$. For a single application of the mechanism, $\sigma \geq \sqrt{2 \log (1.25/\delta)}/\epsilon$ holds for $\epsilon\leq 1$.
The auto-dp package by \cite{wang2019subsampled} computes the relationship between $\epsilon, \delta, \sigma$ numerically, which we use in our method. 





There are two important properties of DP.
The \textit{composability} theorem \cite{dwork2006our} states that the strength of privacy guarantee degrades in a measurable way with repeated use of DP-algorithms. 
This allows us to combine the results of different private mechanisms in \secref{dpmerf_imbalanced} using the advanced composition methods from \cite{pmlr-v89-wang19b}. 
Furthermore,  the \textit{post-processing invariance} property \cite{dwork2006our} tells us that the composition of any data-independent mapping with an $(\epsilon,\delta)$-DP algorithm is also $(\epsilon,\delta)$-DP. This ensures that no analysis of the released synthetic data can yield more information about the real data than what our choice of $\epsilon$ and $\delta$ allows.



What comes next describes our proposal for privacy-preserving data generation. We first present the vanilla version of our algorithm called,  DP-MERF (differentially private mean embeddings with random features). 

\section{Vanilla DP-MERF for unlabeled data}
\label{sec:Methods}

We first introduce the basic version of our DP-MERF algorithm to learn the distribution of an unlabeled dataset. In this setting, we obtain a data generator by minimizing the random feature representation of MMD, given by  %
\begin{align}
    \hat{\vtheta} &= \argmin_\vtheta \widetilde{\mathrm{MMD}}_{rf}^{2}(P_\vx, Q_{\tilde{\vx}_\vtheta})
\end{align} where $P_\vx$ denotes the true data distribution. The samples from $Q$ denoted by $\tilde{\vx}$ are drawn from a generative model $\tilde\vx = G_{\vtheta}(\vz)$. The generative model $G_\theta$ is parameterized by $\theta$ and takes a sample $\vz \sim p(\vz)$ from a known, data-independent distribution as input. 
Using the random Fourier features, we arrive at
\begin{align}\label{eq:MMD_rf_vanilla}
\widetilde{\mathrm{MMD}}_{rf}^{2}(P_\vx, Q_{\tilde{\vx}_\vtheta}) = \bigg\|\widetilde{\vmu}_P-\widehat{\vmu}_Q\bigg\|_{2}^{2} 
\end{align} 
where the random feature mean embedding 
of each distribution is denoted by $\widehat{\vmu}_P = \frac{1}{m}\sum_{i=1}^{m}\hat{\vphi}(\vx_i)$, and $\widehat{\vmu}_Q = \frac{1}{n}\sum_{i=1}^{n}\hat{\vphi}(G_{\vtheta}(\vz_i))$. 

Notice that $\widehat{\vmu}_P$ is the only data-dependent term. Hence, we privatize this term by applying the Gaussian mechanism, defining $\widetilde{\vmu}_P$  by 
\begin{align}\label{eq:noise_up_rf_me}
\widetilde{{\vmu}}_P = \widehat{\vmu}_P  + \Nrm(0, \Delta_{\widehat{\vmu}_P}^2\sigma^2 I)
\end{align}
where the privacy parameter $\sigma$  is chosen as a function of the privacy budget ($\epsilon, \delta$).
The sensitivity of $\widehat{\vmu}_P$ is analytically tractable due to the triangle inequality and the fact that $\|\hat\vphi(\cdot) \|_2 =1$ by construction of the random feature vector given in \eqref{RF}: 
\begin{align}
  \Delta_{\widehat{\vmu}_P} &= \max_{\Dat, \Dat'} \left\| \tfrac{1}{m}\sum_{i=1}^{m}\hat{\vphi}(\vx_i) - \tfrac{1}{m}\sum_{i=1}^{m}\hat{\vphi}(\vx'_i) \right\|_2, \\
  &= \max_{\vx_n, \vx_n'}\left\| \tfrac{1}{m}\hat{\vphi}(\vx_n) - \tfrac{1}{m}\hat{\vphi}(\vx'_n) \right\|_2  \leq \tfrac{2}{m},
\end{align}

Due to the post-processing invariance of DP, we can obtain differentially private generator $G$, since $\widehat{\vmu}_Q$ is data-independent.

\subsection{Bound on the expected absolute error}

If we add noise to the random-feature mean embedding of the data distribution, what is the effect of that noise on the learned generator? Theoretically quantifying this effect is challenging under an arbitrary neural network-based generator. Instead, we theoretically quantify the effect of noise on the objective function.
In particular, given samples $\vx = \{x_i\}_{i=1}^m \sim P$ and $ \tilde{\vx}=\{\tilde{x}_j\}_{j=1}^n \sim Q$, we want to bound the expected absolute error between the noisy random-feature $\mbox{MMD}^2$ (\eqref{MMD_rf_vanilla}) and the original estimator $\mbox{MMD}^2$ (\eqref{MMD_full}).
Given the samples, the error deals with two types of randomness. The first arises due to the random features, $\hat\vphi$. The second arises due to the noise, $\vn$, that we add to the mean-embedding of the data distribution for privacy. 
%
The following proposition formally states the bound to the error (See \suppsecref{proof_of_exp_abs_err} for proof). 
\begin{prop}\label{prop:err_bound}
Given samples $\vx = \{x_i\}_{i=1}^m \sim P$ and $ \tilde{\vx}=\{\tilde{x}_j\}_{j=1}^n \sim Q$, the expected absolute error between the noisy random-feature $\mbox{MMD}^2$ given in \eqref{MMD_rf_vanilla} and the $\mbox{MMD}^2$ given in \eqref{MMD_full} is bounded by 
\begin{align}\label{eq:error_bound}
 &\Em_{\vn}\Em_{\hat{\vphi}} \left[\left|\widetilde{\mathrm{MMD}}_{rf}^{2}(\vx, \tilde{\vx}) - \widehat{\mathrm{MMD}}^2(\vx, \tilde\vx) \right|  \right] 
 \end{align}
 \vspace{-0.5cm}
 \begin{align}\label{eq:upper_bound}
 &\leq \left(\frac{4D \sigma^2 }{m^2}+  \frac{8\sqrt{2} \sigma}{m}  \frac{\Gamma\big((D+1)/2\big)}{\Gamma\big(D/2\big)}\right) + 8\sqrt{\frac{2\pi}{D}}.
\end{align}
\end{prop} where $\Gamma$ is the Gamma function, $\sigma$ is the noise scale (inversely proportional to $\epsilon$), $m$ is the number of training datapoints, and $D$ is the number of features.
\begin{remark}
To prove \propref{err_bound}, we split \eqref{error_bound} into two terms using the triangle inequality.
The first term involves the expected absolute error between the \textbf{noisy} random feature $\mbox{MMD}^2$ (\eqref{MMD_rf_vanilla}) and random feature $\mbox{MMD}^2$ (\eqref{MMD_rf}), which yields the first term (inside a big parenthesis) in \eqref{upper_bound}.
The second term involves the expected absolute error between random feature $\mbox{MMD}^2$ (\eqref{MMD_rf}) and the $\mbox{MMD}^2$ (\eqref{MMD_full}), which yields the second term in \eqref{upper_bound}.
The upper bound is intuitive in that as the number of random features increases, the second term decreases because the random feature MMD is getting closer to MMD, while the first term increases because we add noise to a larger number of random features. 
\end{remark}
\begin{remark}
This bound provides a guideline on how to choose $D$ given a desired privacy level $\epsilon$ and the dataset size $m$. First, given $m$, as long as we choose $D$ such that $m>\sqrt{D}$, the error remains relatively small. 
However, small $D$  can increase the error in the second term (arising from the MMD approximation using random features). Hence, there is a trade-off between these two terms. In our experiments, the datasets we consider have a relatively large $m$ (see \tabref{data_description}), and so choosing a large $D$ ($D \approx 10,000$) incurred a relatively small error for a small value of $\epsilon$.   
\end{remark}

\section{Extension of the vanilla DP-MERF} \label{sec:dpmerf_plus}

After introducing the core functionality of DP-MERF, we extend the vanilla method to cases for 1) labeled data, 2) class-imbalanced data, and 3) heterogeneous data.

\subsection{DP-MERF for labeled data}\label{sec:dpmerf_labeled}

We begin by extending our method to balanced labeled datasets with input features $\vx$ and output labels $\vy$.
In this case, the generator is conditioned on the label: $G_\theta(\vz, \vy) \mapsto \tilde{\vx}$, where $\vy$ is drawn from the uniform distribution over classes.

We encode the class information in the MMD objective, by constructing a kernel from a product of two existing kernels, $k((\vx,\vy), (\vx', \vy')) = k_\vx(\vx, \vx') k_\vy (\vy, \vy')$, where $k_\vx$ is a kernel for input features and $k_\vy$ is a kernel for output labels. 
We choose the Gaussian kernel\footnote{The optimal choice of kernel requires knowledge on the characteristics of the data (see guidelines in Ch. 4 in \cite{williams2006gaussian}). At small data sample sizes, a bad kernel choice will affect the efficiency of the algorithm and can underestimate MMD if the chosen kernel assigns small weights to the “correct” frequencies at which the distributions differ. However, with a large enough sample, any characteristic kernel is able to capture such differences.}  for $k_\vx$ and the polynomial kernel with order-1, $k_\vy(\vy, \vy') = \vy\trp\vy'+c$ for one-hot-encoded labels $\vy$ and set $c=0$.
In this case, the resulting kernel is also characteristic, forming the corresponding MMD as a metric, as explained in \cite{JMLR:v18:17-492}.
We represent the mean embeddings using random features by 
\begin{align}\label{eq:rfME_joint}
  \widehat{\vmu}_{P_{\vx, \vy}} & =\tfrac{1}{m}\sum_{i=1}^{m} \hat{\vf}(\vx_i, \vy_i), \mbox{for true data}\\
  \widehat{\vmu}_{Q_{\vx, \vy}}&= \tfrac{1}{n}\sum_{i=1}^{n} 
    \hat{\vf}(G_{\vtheta}(\vz_i, \vy_i), \vy_i), \mbox{ for synthetic data} \nonumber 
\end{align} 
where we define 
$
\hat{\vf}(\vx_i, \vy_i) := \hat{\vphi}(\vx_i) \vf(\vy_i)\trp, 
$
where $\vf(\vy_i) = \vy_i$ for the order-1 polynomial kernel and $\vy_i$ is one-hot-encoded. 
%
%
See \suppsecref{kernel_product_featmap} for derivation. 
%
With $D$ random features and $C$ classes, the random feature mean embedding in \eqref{rfME_joint} can also be written as 
$
\widehat{\vmu}_{P_{\vx, \vy}} =
\begin{bmatrix}
\vu_1, & \cdots, & \vu_C  \\
\end{bmatrix} \in \mathbb{R}^{D \times C} \nonumber
$
where $c$'th column is given by 
\begin{align}\label{eq:u_c}
\vu_c = \frac{1}{m}\sum_{\vx_i \in X^{(c)}_m}\hat{\vphi}(\vx_i)
\end{align} 
where $X^{(c)}_m$ is the set of the datapoints that belong to the class $c$. 
As in the unlabeled case, $\widehat{\vmu}_{P_{\vx, \vy}}$ has sensitivity $\Delta_{\vmu_P} = \frac{2}{m}$ and is released with the Gaussian mechanism:
\begin{align}\label{eq:privatize_mu}
  \widetilde{\vmu}_{P_{\vx, \vy}} &= \widehat{\vmu}_{P_{\vx, \vy}} + \Nrm(0,  \Delta_{\vmu_P}^2 \sigma^2 \mI_D)
\end{align} 
With the released mean embedding $\widetilde{\vmu}_{P_{\vx, \vy}}$, we construct the private joint maximum mean discrepancy objective:
\begin{align}\label{eq:MMD_rf_joint}
\widetilde{\mathrm{MMD}}_{rf}^{2}(P_{\vx, \vy}, Q_{\tilde\vx_\vtheta, \tilde\vy_\vtheta}) = \bigg\|\widetilde{\vmu}_{P_{\vx, \vy}}-\widehat{\vmu}_{Q_{\vx, \vy}}\bigg\|_{F}^{2},
\end{align}
where $F$ denotes the Frobenius norm.  This kind of objective has been used in the non-private setting \cite{Zhang2019, ijcai2018-293}.


%

%
%

\subsection{DP-MERF for imbalanced data} \label{sec:dpmerf_imbalanced}

Building on the previous section,
%
%
%
%
notice that in \eqref{u_c} the sum in each column is over $m_c$, the number of instances that belong to the particular class $c$, while the divisor is the number of samples in the entire dataset, $m$. This causes difficulties in learning 
when classes are highly imbalanced, as for rare classes $m$ can be significantly larger than the sum of the corresponding column. 
In order to address this problem, we release the vector of class counts, $\vm = [m_1, \cdots, m_C]$ using the Gaussian mechanism: 
\begin{align} \label{eq:privatize_class_counts}
\widetilde{\vm} = \vm +  \Nrm(0,  \Delta_{\vm}^2 \sigma^2 \mI_C)    
\end{align}
As changing a datapoint affects at most two class counts, $\Delta_{\vm} = \sqrt{2}$.
We then modify the released mean embedding by appropriately weighting the embedding for each class: 
%
%
\begin{align} \label{eq:weighted_mean_embedding}
\widetilde{\vmu}^*_{P_{\vx, \vy}}  = 
\begin{bmatrix}
\frac{m}{\widetilde{m}_1}
\widetilde{\vu}_1, & \cdots, & 
\frac{m}{\widetilde{m}_C}
\widetilde{\vu}_C
\end{bmatrix}
\end{align}
Note that we arrive at this expression of mean embedding if we change the kernel on the labels to a weighted one, i.e., $k_\vy(\vy,\vy') = \sum_{c=1}^C\frac{m}{\widetilde{m}_c} \vy_c\trp\vy_c'$. 
%
In the re-weighted mean embedding each class-wise embedding $\frac{m}{\widetilde{m}_c}
\widetilde{\vu}_c$ has a similar norm, and equally contributes to the objective loss.
This ensures that infrequent classes are also modelled accurately.

The total privacy loss results from the composition of the two releases of first $\widetilde{\vm}$ and then $\widetilde{\vmu}_{P_{\vx, vy}}$. 
During training, we sample the generated labels $\widetilde{\vy}$ proportional to the class sizes in $\widetilde{\vm}$. The procedure is summarized in \algoref{rf_ME_joint}.

\begin{algorithm}[!t]
\caption{DP-MERF for imbalanced data}\label{algo:rf_ME_joint}
\begin{algorithmic}
\vspace{0.1cm}
\REQUIRE Dataset $\Dat$, and a privacy level $(\epsilon, \delta)$
\vspace{0.1cm}
\ENSURE $(\epsilon, \delta)$-DP input output samples for all classes\\
\STATE \textbf{Step 1}. Given $(\epsilon, \delta)$, compute the privacy parameter $\sigma$ by the RDP composition in \cite{pmlr-v89-wang19b} for the two uses of the Gaussian mechanism in steps 2 and 3. 
\STATE  \textbf{Step 2}. Release the mean embedding $\widetilde{\vmu}_{P_{\vx, \vy}}$ via \eqref{privatize_mu}
\STATE  \textbf{Step 3}. Release the class counts $\widetilde{\vm}$ using \eqref{privatize_class_counts}.
\STATE  \textbf{Step 4}. Create the weighted mean embedding $\widetilde{\vmu}^*_{P_{\vx, \vy}}$ using \eqref{weighted_mean_embedding}
\STATE  \textbf{Step 5}. Train the generator by minimizing $\widetilde{\mathrm{MMD}}_{rf}^{2}(P_{\vx, \vy}, Q_{\tilde\vx_\vtheta, \tilde\vy_\vtheta}) = \bigg\|\widetilde{\vmu}^*_{P_{\vx, \vy}}-\widehat{\vmu}_{Q_{\vx, \vy}}\bigg\|_{F}^{2}$
\end{algorithmic}
\end{algorithm}


\subsection{DP-MERF for heterogeneous data}\label{sec:Methods_hetero}

To handle heterogeneous data consisting of numerical variables denoted by $\vx_{num}$ and categorical variables denoted by $\vx_{cat}$, we consider the sum of two existing kernels, $k((\vx_{num}, \vx_{cat}), (\vx'_{num}, \vx'_{cat})) = k_{num}(\vx_{num}, \vx'_{num}) + k_{cat} (\vx_{cat}, \vx'_{cat})$, where $k_{num}$ is a kernel for numerical variables and $k_{cat}$ is a kernel for categorical variables. 
Note that this construction of sum of two kernels does not mean that we implicitly assume independence of the two types of variables, for details see \suppsecref{het_var_dependence}.

As before, we could use the Gaussian kernel for $k_{num}(\vx_{num}, \vx'_{num})=\hat{\vphi}(\vx_{num})\trp \hat{\vphi}(\vx'_{num})$ and a normalized polynomial kernel with order-1,  $k_{cat} (\vx_{cat}, \vx'_{cat}) = \frac{1}{d_{cat}} \vx_{cat}\trp\vx_{cat}'$ for one-hot-encoded values $\vx_{cat}$ and the length of $\vx_{cat}$ being $d_{cat}$. This normalization is to match the importance of the two kernels in the resulting mean embeddings. 
Under these kernels, we define
\begin{align}\label{eq:rfME_hetero}
  \widehat{\vmu}_{P_{\vx}} & =\tfrac{1}{m}\sum_{i=1}^{m} \hat{\vh}(\vx_{num}^{(i)}, \vx_{cat}^{(i)}),
\end{align}
%
 \begin{figure*}[htb]
    \centering
    \includegraphics[scale=1.]{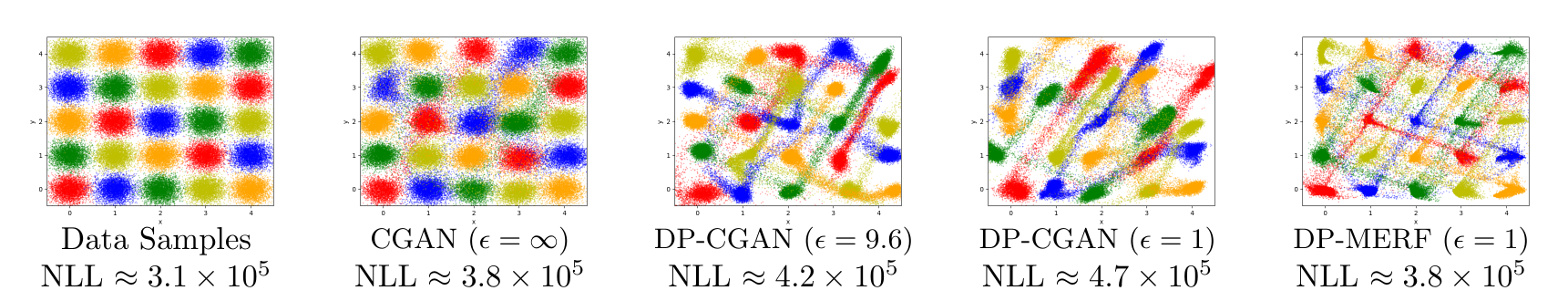}
    \caption{Simulated example from a Gaussian mixture.  \textbf{Left}:  Data samples drawn from a Gaussian Mixture distribution with 5 classes (each color represents a class). NLL denotes the negative log likelihood of the samples given the true data distribution. 
    \textbf{Middle three}:
    Synthetic data generated by DP-CGANs at different privacy levels. CGAN ($\epsilon=\infty$) performs nearly perfectly. However, at $\epsilon=1$, some modes are dropped, which is reflected in NLL. 
    \textbf{Right}: Synthetic data samples generated by DP-MERF at $\epsilon=1$. 
    Our method captures all modes accurately at $\epsilon=1$, which is also reflected in NLL.
    }
    \label{fig:syn2d_samples}
    \vspace{0.4cm}
\end{figure*}

%
where we define $\hat{\vh}(\vx_{num}^{(i)}, \vx_{cat}^{(i)}):=\begin{bmatrix} 
    \hat{\vphi}(\vx_{num}^{(i)})  \\
     \tfrac{1}{\sqrt{d_{cat}}} \vx_{cat}^{(i)}
    \end{bmatrix}$ 
    based on the definition of kernel $k$ (See \suppsecref{kernel_sum_featmap} for derivation). 

In summary, for generating heterogeneous data, we run \algoref{rf_ME_joint} with three changes: 
\begin{enumerate}
    \vspace{-0.3cm}
    \item Redefine $\hat{\vf}(\vx, \vy)$ in \eqref{rfME_joint} as $\hat{\vh}(\vx_{num}, \vx_{cat}) \vf(\vy)\trp$.
    \vspace{-0.3cm}
    \item Redefine  $\vu_c$ in \eqref{u_c} as $ \frac{1}{m}\sum_{i \in X^{(c)}_m}\hat{\vh}(\vx_i)$.
    \vspace{-0.3cm}
    \item Change the sensitivity of $\vu_c$ to $\Delta_{\vu_c} = \frac{2\sqrt{2}}{m}$ (see  \suppsecref{sensitivity_mu_c_heterogeneous} for proof). 
\end{enumerate}


\section{Related work}
\label{sec:related_work}

%

\paragraph{Differentially private data release.}
The field of DP data release contains several distinct lines of research.
As mentioned previously, approaches from a learning theory perspective
\cite{Mohammed:2011:DPD:2020408.2020487, 10.1007/978-3-642-15546-8_11, NIPS2012_4548, 7911185} provide bounds on the utility of the data, but contain either strong assumptions about the types of executed queries or intractable computation, which makes this line of research less relevant to our approach.

Among the query-independent methods,
a large body of work on DP data release focuses on discrete or possible to discretize data. This is a relevant sub-problem in which good results can be achieved by releasing carefully selected marginals of feature subsets, as each feature only takes on a finite set of values.
Such approaches \cite{privbayes, priview, chen2015dpdatapublication} have been, for instance, been dominant among the winning entries of the NIST 2018 Differential Privacy Synthetic Data Challenge \cite{nist2018}, which focused on the task of releasing discrete datasets, utilizing related publicly available data. Although we do not compare to this line of work in the main text, as our method deals with the general setting of DP data release, including continuous data,
we show the comparison to \cite{privbayes} in the \suppsecref{comparison_others}.

The recent line of research into GAN-based private data release
\cite{DPGAN, DP_CGAN, DBLP:conf/sec/FrigerioOGD19, PATE_GAN, gs-wgan} addresses the same general setting and so we select these models for comparison. 
GANs are regarded as a promising model for this task because of their great success in non-private generative modelling and thanks to the fact the generator network of a GAN can be trained without direct access to the data. The GAN discriminator must still be trained with privacy constraints. In most cases, this is achieved through gradient perturbation using DP-SGD, with the exception of PATE-GAN \cite{PATE_GAN}, which is based on the Private Aggregation of Teacher Ensembles (PATE) \cite{papernot:private-training}. 
DP-GAN \cite{DPGAN} and PATE-GAN \cite{PATE_GAN} generate unlabeled data and thus must train one model per class to obtain a labeled dataset. 
DP-CGAN \cite{DP_CGAN} and GS-WGAN \cite{gs-wgan} generate the input features conditioning on the labels, while they do not learn the distribution over the labels. GS-WGAN improves on the basic DP-SGD by alleviating the need for gradient clipping by adapting the loss function and, like PATE-GAN, employs multiple discriminator networks trained on distinct parts of the dataset to amplify privacy by subsampling. We compare these methods with our approach in \secref{experiments}.


\paragraph{Random feature kernel methods with differential privacy.}
Some prior work has employed random feature mean embeddings in the context of differential privacy, but not for the purpose of generative modeling. \cite{BalTolSch18} proposed to use the 
reduced set method in conjunction with random features for sharing DP mean embeddings.
This method performs poorly as the dimension of data grows, which is also noted by the authors (see \suppsecref{comparison_others} for  comparison to our method). 
\cite{sarpatwar2019differentially} also used the random feature representations of mean embeddings for the DP distributed data summarization to take into account covariate shifts.



\begin{table*}[htb]
\caption{Performance comparison on tabular datasets, averaged over five runs. DP-MERF achieves the best scores among private models (bold) on the majority of datasets.}
\vskip 0.1in
\centering
\scalebox{0.9}{
\begin{tabular}{l *{5}{|cc} }
\toprule
& \multicolumn{2}{ c| }{Real} & \multicolumn{2}{ c| }{DP-CGAN}          & \multicolumn{2}{ c| }{DP-GAN}          & \multicolumn{2}{ c| }{\textbf{DP-MERF}} & \multicolumn{2}{ c }{DP-MERF}   \\
& \multicolumn{2}{ c| }{}     & \multicolumn{2}{ c| }{($1,10^{-5}$)-DP} & \multicolumn{2}{ c| }{($1,10^{-5}$)-DP} & \multicolumn{2}{ c| }{($1,10^{-5}$)-DP} & \multicolumn{2}{ c }{non-DP}\\
\midrule
                  & ROC   &  PRC   & ROC    &  PRC  &  ROC   &  PRC   &  ROC   &  PRC   & ROC    &  PRC
\\
\textbf{adult}    & 0.730 & 0.639  &  0.509 & 0.444 &  0.511 & 0.445  &  \textbf{0.650} & \textbf{0.564} &  0.653 & 0.570  \\
\textbf{census}   & 0.747 & 0.415  &  0.655 & 0.216 &  0.529 & 0.166  &  \textbf{0.686} & \textbf{0.358} &  0.692 & 0.369  \\
\textbf{cervical} & 0.786 & 0.493  &  0.519 &\textbf{ 0.200} &  0.485 & 0.183  &  \textbf{0.545} & 0.184 &  0.896 & 0.737  \\
\textbf{credit}   & 0.923 & 0.874  &  0.664 & 0.356 &  0.435 & 0.150  &  \textbf{0.772} & \textbf{0.637} &  0.898 & 0.774  \\
\textbf{epileptic}& 0.797 & 0.617  &  0.578 & 0.241 &  0.505 & 0.196  &  \textbf{0.611} & \textbf{0.340} &  0.616 & 0.335  \\
\textbf{isolet}   & 0.893 & 0.728  &  0.511 & 0.198 &  0.540 & 0.205  &  \textbf{0.547} & \textbf{0.404} &  0.733 & 0.424  \\
\midrule
& \multicolumn{2}{ c| }{F1} & \multicolumn{2}{ c| }{F1} & \multicolumn{2}{ c| }{F1} & \multicolumn{2}{ c| }{F1} & \multicolumn{2}{ c }{F1}\\
\textbf{covtype} & \multicolumn{2}{ c| }{0.643}  & \multicolumn{2}{ c| }{0.285} & \multicolumn{2}{ c| }{\textbf{0.492}} & \multicolumn{2}{ c| }{0.467} & \multicolumn{2}{ c }{0.513} \\
\textbf{intrusion}& \multicolumn{2}{ c| }{0.959} & \multicolumn{2}{ c| }{0.302} & \multicolumn{2}{ c| }{0.251} & \multicolumn{2}{ c| }{\textbf{0.850}}  & \multicolumn{2}{ c }{0.856}\\
\bottomrule
\end{tabular}}\label{tab:summary_all_tabular}
  \vspace{-0.2cm}
\end{table*}

\section{Experiments}\label{sec:experiments}


In this section, we show the robustness of our method on a diverse range of data under strong privacy constraints.
On each dataset, we train DP-MERF and comparison methods to obtain a set of private \textit{synthetic} data samples and compare, how well these emulate the original dataset.
Due to the space limit, we describe all our experimental details (e.g., architecture choices for generators, chosen number of random features, etc.) in the supplementary material. Our code is available at \url{https://github.com/ParkLabML/DP-MERF}. 

\paragraph{2D Gaussian mixtures.}
We begin our experiments on a simple synthetic distribution of Gaussian mixtures which is aligned on a 5 by 5 grid and assigned to 5 classes as shown in \figref{syn2d_samples} (left). The dataset is generated by taking 4000 samples from each Gaussian, reserving 10\% for the test set,
which yields 90000 training samples from the following distribution:
\begin{align}
p(\vx, \vy) = \prod_i^N \sum_{j \in C_{\vy_i}} \frac{1}{C} \mathcal{N}(\vx_i|\vmu_j, \sigma \mI_2)
\end{align}
where $N=90000$, and $\sigma=0.2$. $C=25$ is the number of clusters and $C_y$ denotes the set of indices for means $\vmu$ assigned to class $y$. Five Gaussians are assigned to each class, which leads to a uniform distribution over $\vy$ and $18000$ samples per class.

We choose this dataset because knowing the true data distribution allows us to compute the negative log likelihood (NLL) 
of the samples under the true distribution as a measure of the generated samples' quality:
$\text{NLL}(\vx, \vy) =-\log p(\vx, \vy)$. 
Note that this is different from the other common measure of computing the negative log-likelihood of the true data given the learned model parameters. 

A high NLL score indicates that many samples lie in low density regions of the data distribution. In cases where models tend to under-fit the data, a lower NLL score can thus be regarded as better. However, a low score does not imply that all modes are covered and may also be the result of low sample variance, although the out-of-distribution samples dominate the score, due to the non-linearity of the $\log$ function.

At different levels of privacy, we train DP-CGAN on this dataset and select the models with the fewest dropped modes and secondarily the lowest NLL. We compare this to a DP-MERF model for balanced datasets in \figref{syn2d_samples}.
While DP-CGAN in the non-private setting $(\epsilon=\infty)$ fits the data well, more samples fall out of the distribution as privacy is increased and some modes (like the green one in the top right corner) are dropped. DP-MERF on the other hand preserves all modes and places few samples in low density regions as indicated by the low NLL score. This NLL score is particularly low and on par with the non-private DP-CGAN model, despite a slightly worse fit, because DP-MERF seems to underestimate variance.
%
\begin{table}[!t]
\caption{Tabular datasets. num refers to numerical, cat refers to categorical, and ord refers to ordinal variables}
\label{tab:data_description}
\vskip 0.1in
\centering
\scalebox{0.85}{
\begin{tabular}{l r c c}
\toprule
dataset & $\#$ samps  & $\#$ classes & $\#$ features  \\
\midrule
isolet & 4366 & 2 & 617 num \\
covtype & 406698 &  7 & 10 num, 44 cat \\
epileptic & 11500 & 2 & 178 num \\
credit & 284807 & 2 & 29 num \\
cervical & 753 & 2 & 11 num, 24 cat \\
census & 199523 & 2 & 7 num, 33 cat\\
adult & 22561 & 2 & 6 num, 8 cat\\
intrusion & 394021 & 5 & 8 cat, 6 ord, 26 num\\
\bottomrule\\
\end{tabular}

}
  \vspace{-0.5cm}
\end{table}

\paragraph{Real world data evaluation.}

In the following experiments we do not know the true data distribution and thus require a different method to evaluate the quality of privately generated datasets. Following the common approach used in \cite{PATE_GAN, DP_CGAN, gs-wgan}, we use the private datasets to train a selection of $12$ \textit{predictive models} (see \tabref{credit_best} in the Supplementary for the models). We then evaluate these trained models on a test set of \textit{real} data, which indicates how well the models generalize from the synthetic to the real data distribution and thus how useful the private data would be if used in place of the real data.
Note that hyper-parameters of the 12 models differ because the exact settings used in \cite{PATE_GAN} were not available to us, which means that their scores are not directly comparable to ours.
As comparison models, we test DP-CGAN \cite{DP_CGAN}, as well as our own implementation of an ensemble of 10 DP-GANs, where each model generates data for each class. Our version of DP-GAN differs from \cite{DPGAN} in that it uses standard DP-SGD \cite{DP_SGD} with gradient clipping rather than weight clipping. We further include GS-WGAN \cite{gs-wgan} on image datasets following their original setup.
Note that our DP-GAN implementation and GS-WGAN use the analytical moments accountant \cite{wang2019subsampled} via the autodp package. DP-CGAN uses the RDP accountant \cite{mironov2017renyi} from the tensorflow-privacy package, which is slightly older but still comparable.
The results in \cite{PATE_GAN, DPGAN} could not be reproduced as the released code was incomplete.

As comparison metrics, we use ROC (area under the receiver operating characteristics curve) and PRC (area under the precision recall curve) for binary-labeled data. For multiclass-labeled data we report accuracy for balanced and F1 score for imbalanced data. 
%
As a baseline, we also show the performance of the models trained with the real training data. All the numbers shown in the tables are averages over $5$ independent runs. 
%




\begin{table}[ht!]
\vspace{-0.3cm}
\caption{Test accuracy on image data experiments. DP-MERF at $\epsilon=0.2$ outperforms other methods by a significant margin. $\delta = 10^{-5}$ in all private settings.}
\vskip 0.1in
\centering
\scalebox{0.9}{
\centering
\begin{tabular}{l|c|c}
\toprule
& \textbf{MNIST} & \textbf{FashionMNIST} \\
\midrule
Real data & 0.87 & 0.78   \\
DP-CGAN $\epsilon=9.6$ & 0.50 & 0.39   \\
DP-GAN $\epsilon=9.6$  & 0.48 & 0.46   \\
GS-WGAN  $\epsilon=10$ & 0.53 & 0.50   \\
DP-MERF $\epsilon=1$   & 0.65 & 0.61   \\
DP-MERF $\epsilon=0.2$ & 0.61 & 0.53   \\
\bottomrule
\end{tabular}}\label{tab:summary_all_image}
\vspace{-0.3cm}
\end{table}

\paragraph{Tabular data.}



We explore the extensions of DP-MERF for imbalanced and heterogeneous data on a number of real-world tabular datasets.
%
These datasets contain numerical features with both discrete and continuous values as well as categorical features with either two classes (e.g. whether a person smokes or not) or several classes (e.g. country of origin). The output labels are also categorical and we include datasets with both binary and multi-class labels. \tabref{data_description} summarizes the datasets.
\begin{figure}[ht]
    \centering
    \includegraphics[scale=1.0]{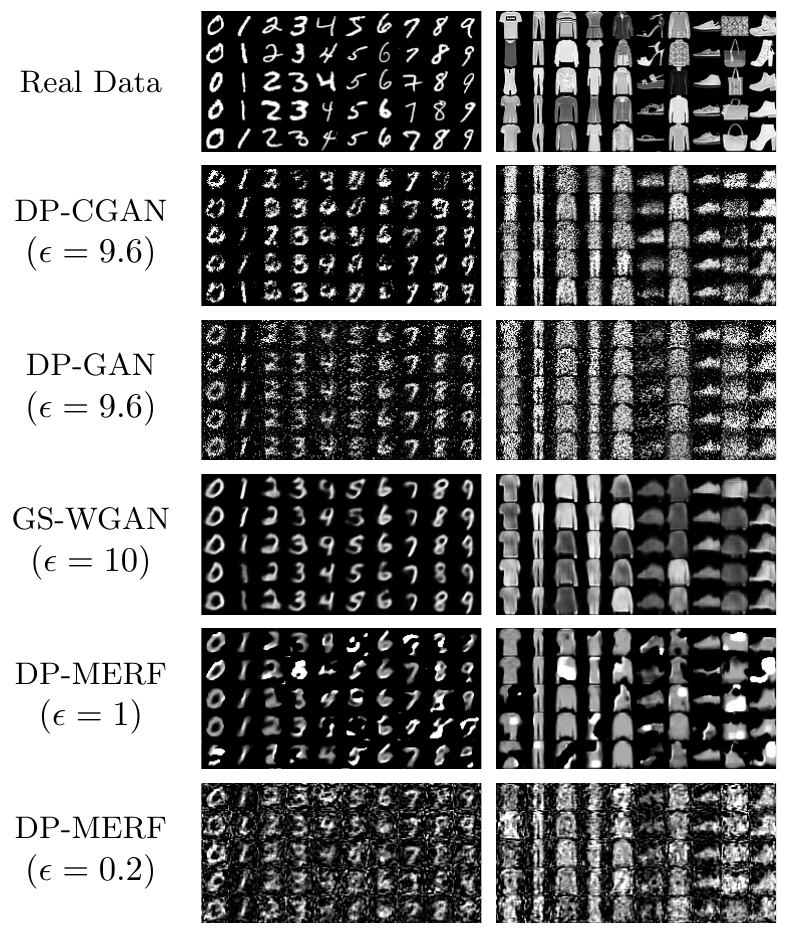}
    \caption{Generated MNIST and FashionMNIST samples from DP-MERF and comparison models with different levels of privacy.}
    \label{fig:generated_samples}
    \vspace{-0.3cm}
\end{figure}
%
%
%
\tabref{summary_all_tabular} shows the average across the $12$ predictive models trained by the generated samples from DP-CGAN, DP-GAN and DP-MERF.
Results for the individual models can be found in \suppsecref{tab_data_details}.
Overall, our method achieved higher values on the evaluation metrics compared to other methods at the same privacy level.

As a side note, the reason the non-private MERF on Cervical data outperforms the real data is due to the small size of the dataset, which is prone to overfitting. Hence, the added sample variance in the generated data has a regularizing effect and improves the performance. 


\paragraph{Image data}

Finally, we evaluate our method on the image datasets, MNIST and FashionMNIST, which are common benchmarks used in \cite{DP_CGAN, DPGAN, gs-wgan}. We apply DP-MERF for balanced data and include convolutional layers, alternating with bi-linear up-sampling, in the generator network to take advantage of the inherent structure of image data.

\tabref{summary_all_image} compares the test accuracy on real data based on generated samples from DP-CGAN, DP-GAN, GS-WGAN and DP-MERF. Results are averaged over 12 classifiers. For the comparison methods, we use the privacy levels reported in the respective papers, as they do not produce usable samples in the high privacy setting at $\epsilon \leq 1$.
It shows that DP-MERF outperforms the GAN based methods by a wide margin and maintains good performance under more meaningful privacy constraints of $(1,10^{-5})$-DP and $(0.2,10^{-5})$-DP. Low overall scores are largely due to the Adaboost and decision tree models which over-fit to the generated data while other models like logistic regression and multi-layer-perceptrions generalize much better. Detailed results are shown in \suppsecref{mnist_data_details}.

In the generated samples of the four tested methods in \figref{generated_samples}, we see that 
the samples from DP-MERF at $\epsilon = 0.2$ are noisier than those of GS-WGAN and DP-CGAN, while still achieving higher downstream accuracy.\footnote{As opposed to the version used in \cite{gs-wgan}, the DP-MERF presented here uses an improved generator architecture and privacy analysis, and outperforms GS-WGAN in the classification tasks.}
This indicates that the distinctive features of the data are preserved despite the noisy appearance of the DP-MERF samples.
In addition, a loss of sample diversity may explain the worse performance of GS-WGAN and DP-CGAN despite higher perceived sample quality, as we already have observed DP-CGAN dropping modes in the Gaussian data experiment.

\section{Summary and Discussion}
We propose a simple and practical algorithm using the random feature representation of kernel mean embeddings for DP data generation. Our method requires a significantly lower privacy budget to produce quality data samples compared to GAN-based approaches, tested on a synthetic dataset, $8$ tabular datasets and $2$ image datasets. The metrics we use are aimed at supervised learning tasks, but the method is not limited to this application. In the future work, we plan to evaluate our method on a more diverse set of tasks and expand it, to scale to more complex data.

\subsection*{Acknowledgments}
We thank Wittawat Jitkrittum, Jia-Jie Zhu, Amin Charusaie and the anonymous reviewers
for their valuable time helping us improve our manuscript.
All three authors are supported by the Max Planck Society.
M. Park and F. Harder are also supported by the Gibs Sch{\"u}le Foundation and the Institutional Strategy of the University of T{\"u}bingen (ZUK63) and the German Federal Ministry of Education and Research (BMBF): T\"ubingen AI Center, FKZ: 01IS18039B.
F. Harder is grateful for the support of the International Max Planck Research School for Intelligent Systems (IMPRS-IS).
K. Adamczewski is grateful for the support of the Max Planck ETH
Center for Learning Systems.

\bibliographystyle{plain}
\bibliography{ms}



\newpage
\onecolumn
\appendix

\begin{center}
    {\LARGE\textbf{Supplementary Material: \\
Differentially Private Random Feature Mean Embeddings for Synthetic Data Generation}}
\end{center}

\section{Background on distance measures for DP data generation}
\label{supp:background_distances}

Many recent papers on DP data generation have utilized the generative adversarial networks (GAN) \cite{NIPS2014_5423} framework, where a discriminator and a generator play a  min-max form of game to optimize for 
the \textit{Jensen-Shannon divergence} between the true and  synthetic data distributions \cite{DP_EM, DP_CGAN, PATE_GAN}.
The Jensen-Shannon divergence belongs to the family of divergences, known as \textit{Ali-Silvey distance}, \textit{Csisz\'ar's $\phi$-divergence} \cite{CIT-004}, defined as $D_{\phi}(P,Q) = \int_{M} \phi \left(\frac{P}{Q}\right) dQ$ where $M$ is a measurable space and $P,Q$ are probability distributions.
Depending on the form of $\phi$, $D_{\phi}(P,Q)$ recovers popular divergences\footnote{See Table 1 in \cite{nowozin2016} for various $\phi$ divergences in the context of GANs.} such as the Kullback-Liebler (KL) divergence ($\phi(t)=t\log t$).

Another popular family of distance measure is \textit{integral probability metrics (IPMs)}, which is defined by 
$D(P,Q) = \mbox{sup}_{f \in \mathcal{F}} \left| \int_{M} f d P - \int_{M} f dQ \right|$ where  $\mathcal{F}$ is a class of real-valued bounded measurable functions on $M$. 
Depending on the class of functions, there are several popular choices of IPMs.
%
For instance, when $\mathcal{F} = \{f: \Vert f \Vert_L \leq 1 \}$, where $\Vert f \Vert_L := \mbox{sup}\{|f(x)-f(y)|/\rho(x,y): x\neq y \in M \}$ for a metric space $(M, \rho)$, 
$D(P,Q)$ yields the \textit{Kantorovich} metric, and when $M$ is separable, the Kantorovich metric 
recovers the
\textit{Wasserstein} distance, a popular choice for generative modelling such as Wasserstein-GAN and Wasserstein-VAE \cite{Arjovsky2017WassersteinG, tolstikhin2018wasserstein}. The GAN framework with the Wasserstein distance was also used for DP data generation \cite{DPGAN, DBLP:conf/sec/FrigerioOGD19}.

As another example of IPMs, when $\mathcal{F} = \{f: \Vert f \Vert_{\mathcal{H}} \leq 1 \}$, i.e., the function class is a unit ball in reproducing kernel Hilbert space (RKHS) ${\mathcal{H}}$ associated with a positive-definite kernel $k$,  $D(P,Q)$ yields the \textit{maximum mean discrepancy} (MMD), 
$MMD(P,Q) =  \mbox{sup}_{f \in \mathcal{F}} \left| \int_{M} f d P - \int_{M} f dQ \right|$.
In this case finding a supremum is analytically tractable and the solution is represented by the difference in the mean embeddings of each probability measure: $MMD(P,Q) =  \Vert \mu_{P} - \mu_{Q} \Vert_{H}$, where 
$ \mu_{P}  = \mathbb{E}_{\vx \sim \mathbb{P}}[k(\vx, \cdot)]$ and
$ \mu_{\mathbb{Q}}  = \mathbb{E}_{\vy \sim \mathbb{Q}}[k(\vy, \cdot)]$. 
For a characteristic kernel $k$, the squared MMD forms a metric, i.e., $MMD^2  = 0$, if and only if $P=Q$. 
%
MMD is also a popular choice for generative modelling in the GAN frameworks \cite{mmd_gan}, as MMD compares two probability measures in terms of all possible moments (no information loss due to a selection of a certain set of moments); and the MMD estimator is in closed form (\eqref{MMD_full}) and easy to compute by the pair-wise evaluations of a kernel function using the points drawn from $P$ and $Q$.
 
In this work, we propose to use a particular form of MMD via \textit{random Fourier feature} representations \cite{rahimi2008random} of kernel mean embeddings for DP data generation.


\section{Derivation of the bound on the expected absolute error}
\label{supp:proof_of_exp_abs_err}

Given the samples drawn from two probability distributions: $X_{m}=\{x_{i}\}_{i=1}^{m} \sim P$ and $X'_{n}=\{x'_{i}\}_{i=1}^{n} \sim Q$, the biased MMD estimator is given by \cite{Gretton2012}:
\begin{align} 
 \widehat{\mathrm{MMD}}^2(X_{m},X'_{n}) &= \tfrac{1}{m^2}\sum_{i,j=1}^{m}k(x_{i},x_{j}) +\tfrac{1}{n^2}\sum_{i,j=1}^{n}k(x'_{i},x'_{j}) 
 -\tfrac{2}{mn}\sum_{i=1}^{m}\sum_{j=1}^{n}k(x_{i},x'_{j}).
\end{align}
%
The MMD estimator using the $D$-dimensional random Fourier features $\hat\vphi$ for the mean embeddings  $\widehat{\vmu}_P = \frac{1}{m}\sum_{i=1}^{m}\hat{\vphi}(\vx_i)$ and $\widehat{\vmu}_Q = \frac{1}{n}\sum_{i=1}^{n}\hat{\vphi}(G_{\vtheta}(\vz_i))$  is defined as 
\begin{align} 
\widehat{\mathrm{MMD}}_{rf}^{2}(P,Q)=\bigg\|\widehat{\vmu}_P - \widehat{\vmu}_Q \bigg\|_{2}^{2}.
\end{align}
The noisy MMD is given by 
\begin{align}  
\widetilde{\mathrm{MMD}}_{rf}^{2}(P_\vx, Q_{\tilde{\vx}_\vtheta}) = \bigg\|\widetilde{\vmu}_P-\widehat{\vmu}_Q\bigg\|_{2}^{2}, 
\end{align} 
where $\widetilde{\vmu}_P$  is given by 
\begin{align}  
\widetilde{{\vmu}}_P = \widehat{\vmu}_P  + \vn
\end{align} where $\vn$ is a draw from a Gaussian distribution $\vn \sim \Nrm(0, \Delta_{\widehat{\vmu}_P}^2\sigma^2 I)$. Note that for the bounded kernels with bound 1, $\Delta_{\widehat{\vmu}_P} = \frac{2}{m}$. 

Now the proposition is given as follows. 
\begin{prop}
Given samples $\vx = \{x_i\}_{i=1}^m \sim P$ and $ \tilde{\vx}=\{\tilde{x}_j\}_{j=1}^n \sim Q$, the expected absolute error between the noisy random-feature (squared) MMD defined in \eqref{MMD_rf_vanilla} and the squared MMD \eqref{MMD_full} is bounded by 
\begin{align}
 &\Em_{\vn}\Em_{\hat\vphi} \left[\left|\widetilde{\mathrm{MMD}}_{rf}^{2}(\vx, \tilde{\vx}) - \widehat{\mathrm{MMD}}^2(\vx, \tilde\vx) \right|  \right], \\
 &\leq \left(\frac{4D \sigma^2 }{m^2}+  \frac{8\sqrt{2} \sigma}{m}  \frac{\Gamma\big((D+1)/2\big)}{\Gamma\big(D/2\big)}\right) + 8\sqrt{\frac{2\pi}{D}}
\end{align} 
\end{prop} where $\Gamma$ is the Gamma function.

To prove this proposition, we first rewrite the absolute error in terms of two terms due to the triangle inequality:
\begin{align} 
 &\Em_{\vn}\Em_{\hat\vphi} \left[\left|\widetilde{\mathrm{MMD}}_{rf}^{2}(\vx, \tilde{\vx}) - \widehat{\mathrm{MMD}}^2(\vx, \tilde\vx) \right|  \right] \nonumber \\
 &\leq \Em_{\vn}\Em_{\hat\vphi} \left[\left|\widetilde{\mathrm{MMD}}_{rf}^{2}(\vx, \tilde{\vx}) - \widehat{\mathrm{MMD}}_{rf}^2(\vx, \tilde\vx) \right|  \right]  + 
 \Em_{\hat{\vphi}}\left[\left|\widehat{\mathrm{MMD}}_{rf}^{2}(\vx, \tilde{\vx}) - \widehat{\mathrm{MMD}}^2(\vx, \tilde\vx) \right|  \right]. 
\end{align} 
What follows next proves each of these terms. 

\subsection{Randomness due to random features} 
We restate the result of [Sec.\ 3.3 of Sutherland and Schneider 2016]. 
\begin{lem}[Sec.\ 3.3 of Sutherland and Schneider 2016]\label{lemma:MMD_UB}
Given samples $\vx = \{x_i\}_{i=1}^n \sim P$ and $ \tilde{\vx}=\{\tilde{x}_j\}_{j=1}^m \sim Q$, the probabilistic bound between the approximate MMD with random features, denoted by $\widehat{\mathrm{MMD}}_{rf}(\vx, \tilde{\vx})$ and the original MMD, denoted by $\widehat{\mathrm{MMD}}(\vx, \tilde{\vx})$, holds 
\begin{align}
    &\mathbb{P}\left[\left|\widehat{\mathrm{MMD}}_{rf}^{2}(\vx, \tilde{\vx}) - \widehat{\mathrm{MMD}}^2(\vx, \tilde{\vx}) \right| \geq t_1 \right] \leq 2 \exp\left(-\frac{1}{128}D t_1^2\right) := U_1,
    \end{align}
where the randomness comes from the random features, and $\mathbb{E}_{\hat\vphi} [ \widehat{\mathrm{MMD}}_{rf}(\vx, \tilde{\vx}) ] = \mathrm{MMD}(\vx, \tilde{\vx})$. 
\end{lem}
\begin{proof}
To prove the proposition, we first consider the mean map kernel (MMK) defined by 
\begin{align}
    \mathrm{MMK}(\vx, \tilde\vx) = \frac{1}{nm} \sum_{i=1}^n\sum_{j=1}^m k(x_i, \tilde{x}_j) \approx \mathrm{MMK}_{\hat\vphi}(\vx, \tilde\vx):=\hat\vphi(\vx)\trp \hat\vphi(\tilde\vx),
\end{align} which can be approximated by the random feature representations, denoted by $\mathrm{MMK}_{\hat\vphi}(\vx, \tilde\vx)$. The random feature mean-embedding of $P$ is denoted by  $\hat{\vphi}(\vx)$. 
Similarly, we can define $\mathrm{MMK}(\vx,\vx)$ and $\mathrm{MMK}(\tilde\vx, \tilde\vx)$, and define MMD in terms of MMKs
\begin{align}\label{eq:MMD_2}
    \widehat{\mathrm{MMD}}^2(\vx, \tilde\vx) = \mathrm{MMK}(\vx,\vx) + \mathrm{MMK}(\tilde\vx, \tilde\vx) - 2 \mathrm{MMK}(\vx, \tilde\vx).
\end{align}
Notice that when we use the cosine/sine representation of random features, changing the frequency $\omega_k$ to $\hat\omega_k$ causes a bounded difference in the $k$th coordinate of the MMK estimate, $\mathrm{MMK}_\vphi(\vx, \tilde\vx)$:
\begin{align}
    \left|\frac{1}{nm}\sum_{i=1}^n\sum_{j=1}^m \frac{2}{D} \left[\cos((\omega_k\trp(x_i -\tilde{x}_j)) - \cos((\omega'_k\trp(x_i -\tilde{x}_j)) \right] \right| \leq \frac{4}{D}.
\end{align}
Due to this bounded difference in each coordinate of random feature MMK, we can compute the tail bound using the McDiarmid's inequality,
\begin{align}
        \mbox{Pr}\left[\left|\mathrm{MMK}_{\hat\vphi}(\vx, \tilde{\vx}) - \mathrm{MMK}(\vx, \tilde{\vx}) \right| \geq t_1 \right] \leq 2 \exp\left(-\frac{1}{8}D t_1^2\right).
\end{align}
Now using the definition of $\mathrm{MMD}^2$ given in \eqref{MMD_2}, we obtain the tail bound.
\begin{align}
    \mbox{Pr}\left[\left|\widehat{\mathrm{MMD}}_{rf}^{2}(\vx, \tilde{\vx}) - \widehat{\mathrm{MMD}}^2(\vx, \tilde{\vx}) \right| \geq t_1 \right] \leq 2 \exp\left(-\frac{1}{128}D t_1^2\right). 
\end{align}
\end{proof}

As a result of \lemmaref{MMD_UB}, the expected absolute error of the random-feature MMD is bounded by
\begin{lem}[Sec.\ 3.3 of Sutherland and Schneider 2016]\label{lemma: mean_square_error_mmd}
Given samples $\vx = \{x_i\}_{i=1}^n \sim P$ and $ \tilde{\vx}=\{\tilde{x}_j\}_{j=1}^m \sim Q$, the probabilistic bound between the approximate MMD with random features, denoted by $\widehat{\mathrm{MMD}}_{rf}(\vx, \tilde{\vx})$ and the original MMD, denoted by $\mathrm{MMD}(\vx, \tilde{\vx})$, holds 
\begin{align}\label{eq:exp_moment}
    \Em_{\hat\vphi}\left[\left|\widehat{\mathrm{MMD}}_{rf}^{2}(\vx, \tilde{\vx}) - \widehat{\mathrm{MMD}}^2(\vx, \tilde{\vx}) \right| \right] \leq  8 \sqrt{2\pi/D}.
    \end{align}
\end{lem}
\begin{proof}
For a non-negative random variable, $\left|\widetilde{\mathrm{MMD}}_{rf}^{2}(P, Q) - \widehat{\mathrm{MMD}}^2(P, Q) \right| $
\begin{align}
    \Em_{{\hat\vphi}}\left[\left|\widehat{\mathrm{MMD}}_{rf}^{2}(\vx, \tilde{\vx}) - \widehat{\mathrm{MMD}}^2(\vx, \tilde{\vx}) \right| \right] &= \int_0^\infty \mbox{Pr}\left[ \left|\widehat{\mathrm{MMD}}_{rf}^{2}(\vx, \tilde{\vx}) - \widehat{\mathrm{MMD}}^2(\vx, \tilde{\vx}) \right| \geq t_1 \right] dt_1, \\
    &\leq 2 \int_0^\infty \exp\left(-\frac{1}{128}D t_1^2\right) dt_1, \mbox{due to \lemmaref{MMD_UB}}, \\
    & \; = 8\sqrt{\frac{2\pi}{D}},  \mbox{due to the Gaussian integral}.
\end{align}
\end{proof}

\subsection{Randomness due to noise for privacy} 
The following remark bound the first moment of the privatized MMD proxy $\widetilde{\mathrm{MMD}}_{rf}$ and the MMD proxy $\widehat{{\rm MMD}}_{rf}$.
\begin{lem}
Let $\widetilde{\mathrm{MMD}}_{rf}(\vx, \tilde\vx) := \| \widehat{\vmu}_P(\vx) + \vn - \widehat{\vmu}_Q(\tilde\vx) \|_2 $, where $\vn \sim \Nrm(0, \sigma^2 \Delta_{\widehat{\vmu}_P}^2 I_D)$. Also, let $\widehat{\mathrm{MMD}}_{rf}(\vx, \tilde\vx) := \| \widehat{\vmu}_P(\vx) - \widehat{\vmu}_Q(\tilde\vx) \|_2 $.  
Then,
\begin{align}
    \Em_{\vn}\Em_{\hat\vphi}\Big[\big|\widetilde{\mathrm{MMD}}^2_{rf}(\vx, \tilde\vx)-\widehat{{\rm MMD}}^2_{rf}(\vx, \tilde\vx)\big|\Big]
   \leq \frac{D \sigma^2}{m^2}+  4\sqrt{2} \sigma \frac{\Gamma\big((D+1)/2\big)}{m\Gamma\big(D/2\big)}
\end{align}
\end{lem}
\begin{proof}  
\begin{align}
    \Em_{\vn}\Em_{\hat\vphi}\Big[\big|\widetilde{\mathrm{MMD}}^2_{rf}(\vx, \tilde\vx)-\widehat{{\rm MMD}}^2_{rf}(\vx, \tilde\vx)\big|\Big]
    &\overset{(a)}{=} \Em_{\hat\vphi}\left[\Em_{\vn}\bigg[ \left| \vn\trp\vn + 2 \vn\trp( \widehat{\vmu}_P(\vx) - \widehat{\vmu}_Q(\tilde\vx)) \right|\bigg]\right], \\
    &\overset{(b)}{\leq} \Em_{\hat\vphi}\left[\Em_\vn\Big[\vn\trp\vn\Big]+2\Em_\vn\bigg[\Big|\vn\trp( \widehat{\vmu}_P(\vx) - \widehat{\vmu}_Q(\tilde\vx))\Big| \bigg]\right], \nonumber\\
    &\overset{(c)}{=} D \sigma^2 \Delta_{\widehat{\vmu}_P}^2+ 2\sqrt{2} \Em_{\hat\vphi}\left[\|\widehat{\vmu}_P(\vx)-\widehat{\vmu}_Q(\vx)\|_2\right] \sigma \Delta_{\widehat{\vmu}_P}\frac{\Gamma\big((D+1)/2\big)}{\Gamma\big(D/2\big)}, \\
    &\overset{(d)}{=} \frac{D \sigma^2}{m^2}+  4\sqrt{2} \sigma \frac{\Gamma\big((D+1)/2\big)}{m\Gamma\big(D/2\big)},
    \label{eq:rf_noisy}
\end{align}
\end{proof}
where $(a)$ is by expanding two terms following their definitions:
$
\widetilde{\mathrm{MMD}}_{rf}^{2}(\vx, \tilde\vx) - \widehat{\mathrm{MMD}}_{rf}^{2}(\vx, \tilde\vx) =
\vn\trp\vn + 2 \vn\trp( \widehat{\vmu}_P(\vx) - \widehat{\vmu}_Q(\tilde\vx)).
$
$(b)$ is followed by triangle inequality.  
$(c)$ is followed by the second moment of the chi-square random variable (first term) and the first moment of the chi distribution (second term). 
$(d)$ is 
by taking the maximum over random features.  Under the random feature representation we use in our paper, the L2-norm of random features is bounded by 1. Hence,  
$\Em_{\hat\vphi}\left[\|\widehat{\vmu}_P(\vx)-\widehat{\vmu}_Q(\vx)\|_2\right] \leq \max_{\hat{\vphi}}\left[\|\widehat{\vmu}_P(\vx)-\widehat{\vmu}_Q(\vx)\|_2\right] \leq \max_{\hat{\vphi}}\left[\|\widehat{\vmu}_P(\vx)\|_2 +\|\widehat{\vmu}_Q(\vx)\|_2\right] \leq 1+1 =2$. 
 

\section{Derivation of feature maps for a product of two kernels}
\label{supp:kernel_product_featmap}
Under our assumption, we decompose the kernel below into two kernels:
\begin{align}
    & k((\vx,\vy), (\vx', \vy')) \nonumber \\
    &= k_\vx(\vx, \vx') k_\vy (\vy, \vy'), \mbox{ product of two kernels} \nonumber \\
    &\approx \left[\hat{\vphi}(\vx')\trp \hat{\vphi}(\vx)\right] \; \left[\vf(\vy)\trp \vf(\vy')\right],  \mbox{ random features for kernel } k_\vx \nonumber \\
    &= \mbox{Tr}\left(\hat{\vphi}(\vx')\trp \hat{\vphi}(\vx) \vf(\vy)\trp \vf(\vy')\right), \nonumber \\
    &= \mbox{vec}( \hat{\vphi}(\vx') \vf(\vy')\trp)\trp \mbox{vec}( \hat{\vphi}(\vx) \vf(\vy)\trp) = \hat{\vf}(\vx', \vy') \trp \hat{\vf}(\vx, \vy) \nonumber 
\end{align}

\section{Derivation of feature maps for a sum of two kernels}
\label{supp:kernel_sum_featmap}
Under our assumption, we compose the kernel below from the sum of two kernels:
\begin{align}
     &k((\vx_{num}, \vx_{cat}), (\vx'_{num}, \vx'_{cat})) \nonumber \\
     &= k_{num}(\vx_{num}, \vx'_{num}) + k_{cat} (\vx_{cat}, \vx'_{cat}), \nonumber \\
    &\approx \hat{\vphi}(\vx_{num})\trp \hat{\vphi}(\vx'_{num}) + \tfrac{1}{\sqrt{d_{cat}}} \vx_{cat}\trp\vx_{cat}', \nonumber \\
    &= \begin{bmatrix} 
    \hat{\vphi}(\vx_{num})  \nonumber \\
     \frac{1}{\sqrt{d_{cat}}} \vx_{cat}
    \end{bmatrix}^T
    \begin{bmatrix} 
    \hat{\vphi}(\vx_{num})  \nonumber \\
     \frac{1}{\sqrt{d_{cat}}} \vx_{cat}
    \end{bmatrix} \nonumber  \\
    &= \hat{\vh}(\vx_{num}, \vx_{cat})^T \hat{\vh}(\vx_{num}, \vx_{cat}).\nonumber
\end{align}

\section{Sensitivity of class counts}
\label{supp:weight_sensitivity}

Consider the vector of class counts $\vm = [m_1, \cdots, m_C  ],$  where each element $m_c$ is the number of samples with class $c$ in the dataset. The class counts of two neighbouring datasets $\Dat$ and $\Dat' = (\Dat \setminus \{\vx\}) \cup \{\vx'\}$ can differ in at most two entries $k,l$ and at most by 1 in either entry. Assuming $\vy \neq \vy'$, then for $\vy_k = 1$, $m_k = m'_k + 1$ and for $\vy'_l = 1$, $m'_l = m_l + 1$ and $m_i = m'_i$ in all other cases. If $\vy = \vy'$, then $\vm = \vm'$.
Letting $\vm$ and $\vm'$ denote the class counts of $\Dat$ and $\Dat'$ respectively, we get the following:
\begin{align}
\Delta_{\vm} &=  \max_{\Dat, \Dat'} \left\| \vm - \vm' \right\|_2
= \max_{\Dat, \Dat'} \sqrt{\sum_{i=1}^C m_i - m'_i} = \sqrt{2}
\end{align}




\section{Sensitivity of \texorpdfstring{$\hat{\vmu}_P$}{mu-p} with homogeneous data}
\label{supp:sensitivity_mu_c_homogeneous}

Below, we show that the sensitivity of the data mean embedding for homogeneous labeled data is the same as for unlabeled data. In order, we first use the fact that $\Dat$ and $\Dat'$ are neighbouring, which implies that $m-1$ of the summands on each side cancel and we are left with the only distinct datapoints, which we denote as $(\vx, \vy)$ and $(\vx', \vy')$. We then apply the triangle inequality and the definition of $\vf$. As $\vy$ is a one-hot vector, all but one column of $\hat{\vphi}(\vx) \vy\trp$ are 0, so we omit them in the next step and finally use that $\| \hat{\vphi}(\vx) \|_2 = 1$.

\begin{align}
\Delta_{\hat{\vmu}_P} &= \max_{\Dat, \Dat'} \left\| \tfrac{1}{m}\sum_{(\vx_i, \vy_i) \in \Dat}\hat{\vf}(\vx_i, \vy_i) - \tfrac{1}{m}\sum_{(\vx'_i, \vy'_i) \in \Dat'}\hat{\vf}(\vx'_i, \vy'_i) \right\|_F \\
&= \max_{(\vx, \vy), (\vx', \vy)} \left\| \tfrac{1}{m}\hat{\vf}(\vx, \vy) - \tfrac{1}{m}\hat{\vf}(\vx', \vy') \right\|_F \\
&\leq \max_{(\vx, \vy)} \tfrac{2}{m} \left\| \hat{\vf}(\vx, \vy) \right\|_F \\
&= \max_{(\vx, \vy)} \tfrac{2}{m} \left\| \hat{\vphi}(\vx) \vy\trp \right\|_F \\
&= \max_{\vx} \tfrac{2}{m} \left\| \hat{\vphi}(\vx) \right\|_2 \\
&= \frac{2}{m}
\end{align}


\section{Sensitivity of \texorpdfstring{$\vmu_P$}{mu-P} with heterogeneous data}
\label{supp:sensitivity_mu_c_heterogeneous}

In the case of heterogeneous data, recall that
$\hat{\vh}(\vx^{(i)}_{num}, \vx^{(i)}_{cat})=\begin{bmatrix} 
    \hat{\vphi}(\vx_{num}^{(i)}) \\
     \frac{1}{\sqrt{d_{cat}}} \vx_{cat}^{(i)}
    \end{bmatrix}
$ and $\vmu_P = \frac{1}{m}\sum_{(\vx_i, \vy_i) \in \Dat}\hat{\vh}(\vx_i)\vy_i\trp$ where $\vx_i $ is the concatenation of $\vx_{num}^{(i)}$ and $\vx_{cat}^{(i)}$.
Analogous to the homogeneous case, we 
first derive that the labeled and unlabeled embedding have the same sensitivity (in \eqref{mu_het_middle}). We apply the definition of $\hat{\vh}$ and analyze the numerical and categorical parts separately, using the facts that $\| \hat{\vphi}(\vx) \|_2 = 1$ and, since $\vx_{cat}$ is binary, $\| \vx_{cat} \|_2 \leq \sqrt{d_{cat}}$.

\begin{align}
\Delta_{\vmu_P} &= 
\max_{\Dat, \Dat'} \left\| \tfrac{1}{m}\sum_{(\vx_i, \vy_i) \in \Dat}\hat{\vh}(\vx_i)\vy_i\trp - \tfrac{1}{m}\sum_{(\vx'_i, \vy'_i) \in \Dat'}\hat{\vh}(\vx'_i) \vy'_i\trp \right\|_F \\
&=
\max_{(\vx, \vy), (\vx', \vy')} \left\| \tfrac{1}{m} \hat{\vh}(\vx_i)\vy_i\trp - \tfrac{1}{m} \hat{\vh}(\vx'_i) \vy'_i\trp \right\|_F \\
&\leq
\max_{(\vx, \vy)} \tfrac{2}{m} \left\|  \hat{\vh}(\vx)\vy\trp \right\|_F \\
&=
\max_{\vx} \tfrac{2}{m} \left\|  \hat{\vh}(\vx) \right\|_2 \label{eq:mu_het_middle}\\
&=
\max_{\vx} \tfrac{2}{m} \left\|  \begin{bmatrix} 
\hat{\vphi}(\vx_{num}) \\
\frac{1}{\sqrt{d_{cat}}} \vx_{cat}
\end{bmatrix}
\right\|_2 \\
&=
\max_{\vx} \tfrac{2}{m} 
\sqrt{
\|\hat{\vphi}(\vx_{num})\|_2^2 + 
\|\tfrac{1}{\sqrt{d_{cat}}} \vx_{cat}\|_2^2
} \\
&=
\tfrac{2}{m}\sqrt{1 + \tfrac{d_{cat}}{d_{cat}}} \\
&= \frac{2\sqrt{2}}{m}
\end{align}

%

\setcounter{section}{8}  




\section{Variables in heterogeneous data are not treated as independent}
\label{supp:het_var_dependence}

While the impression may arise, our method does not assume independence between the continuous and the discrete variables, but models correlations between the two types of variables implicitly. With the sum of two kernels, the embedding is a concatenation of the two: $[E_x \phi_x(x), E_y \phi_y(y)]$, where $E_x$ means expectation wrt $p(x)$ and $E_y$ is wrt $p(y)$. To compute $p(x)$, we need $p(y)$ with which we marginalize out $y$, as $p(x) = \int p(x,y) dy$. This marginalization implicitly takes into account the correlation between the two. This is less explicit than the case using the product of two kernels. However, the sum kernel is chosen for computational tractability: a sum kernel in Fourier representation has $d_x+d_y$ features while a product kernel has $d_x \cdot d_y$.

\setcounter{section}{10}





\section{Heterogeneous and homogenous tabular data}
\label{supp:tab_data_details}

In this section we describe the tabular datasets we have used in our experiments with their respective sources. We include the details of data preprocessing in case it was performed on a dataset. The datasets in this form were used in all our experiments as well as the experiments on the benchmark methods.

\subsubsection*{Credit}

Credit card fraud detection dataset contains the categorized information of credit card transactions which were either fraudelent or not. The dataset comes from a Kaggle competition and is available at the source,
\url{https://www.kaggle.com/mlg-ulb/creditcardfraud}. The original data has 284807 examples, of which negative samples are 284315 and positive 492. The dataset has 31 categories, 30 numerical features and a binary label. We used all but the first feature (Time). 

\subsubsection*{Epileptic}

Epileptic dataset describes brain activity with numerical features being EEG recording at a different point in time. The dataset comes from the UCI database, \url{https://archive.ics.uci.edu/ml/datasets/Epileptic+Seizure+Recognition}. It contains 11500 data points, and 179 categories, 178 features and a label. The original dataset contains five different labels which we binarize into two states, seizure or no seizure. Thus, there are 9200 negative samples and 2300 positive samples.

\subsubsection*{Census}
The dataset can be downloaded by means of SDGym package, \url{https://pypi.org/project/sdgym/}. The dataset has 199523 examples, 187141 are negative and 12382 are positive. There are 40 categories and a binary label. This dataset contains 7 numerical and 33 categorical features.

\subsubsection*{Intrusion}

The dataset was used for The Third International Knowledge Discovery and Data Mining Tools Competition held at the Conference on Knowledge Discovery and Data Mining, 1999, and can be found at \url{http://kdd.ics.uci.edu/databases/kddcup99/kddcup99.html}. We used the file, kddcup.data\_10\_percent.gz. It is a multi-class dataset with five labels describing different types of connection intrusions. The labels were first grouped into five categories and due to few examples, we restricted the data to the top four categories.

\subsubsection*{Adult}

The dataset contains information about people's attributes and their respective income which has been thresholded and binarized. It has 22561 examples, and 14 features and a binary label. The dataset can be downloaded by means of SDGym package,\url{https://pypi.org/project/sdgym/}.

\subsubsection*{Isolet}

The dataset contains sound features to predict a spoken letter of alphabet. The inputs are sound features and the output is a latter. We binaried the labels into two classes, consonants and vowels. The dataset can be found at \url{https://archive.ics.uci.edu/ml/datasets/isolet}

\subsubsection*{Cervical}

This dataset is created with the goal to identify the risk factors associated with cervical cancer. It is the smallest dataset with 858 instances, and 35 attributes, of which The data can be found at 15 are numerical 24 are categorical (binary). The dataset can be found at \url{https://archive.ics.uci.edu/ml/datasets/Cervical+cancer+\%28Risk+Factors\%29}. The data, however, contains missing data. We followed the pre-processing suggested at \url{https://www.kaggle.com/saflynn/cervical-cancer-lynn} and further removed the data with the most missing values and replaced the rest with the category mean value.

\subsubsection*{Covtype}

The dataset describes forest cover type from cartographic variables.
The data can be found at \url{https://archive.ics.uci.edu/ml/datasets/covertype}. It contains 53 attributes and a multi-class label with 7 classes of forest cover types.

\subsection{The training}

We provide here the details of training procedure. Some of the datasets are very imbalanced, that is they contain much more examples with one label over the others. In attempt of making categories more balanced, we undersampled the class with the largest number of samples. The complexity of a dataset also determined the number of Fourier features we used. We also varied the batch size (we include the fraction of dataset used in a batch), and the number of epochs in the training. We provide the detailed parameter settings for each of the dataset in the following table.

\begin{table*}[htb]
\caption{Parameters settings for training tabular datasets}
\centering
\scalebox{0.9}{
 \begin{tabular}{c|ccc|ccc|c}
\toprule
& \multicolumn{3}{ c| }{\textbf{non-private}} & \multicolumn{3}{ c| }{\textbf{private}} & \\
& \multicolumn{1}{ c }{\#} & \multicolumn{1}{ c }{mini-batch}  & \multicolumn{1}{ c| }{\# Fourier} 
& \multicolumn{1}{ c }{\#} & \multicolumn{1}{ c }{mini-batch}  & \multicolumn{1}{ c| }{\# Fourier} 
& \multicolumn{1}{ c }{undersampling}  \\
& \multicolumn{1}{ c }{epochs} &  \multicolumn{1}{ c }{size} &  \multicolumn{1}{ c| }{features}
& \multicolumn{1}{ c }{epochs} &  \multicolumn{1}{ c }{size} &  \multicolumn{1}{ c| }{features}
& \multicolumn{1}{ c }{rate}
\\ \midrule
adult & 8000 & 0.1 & 50000 & 8000 & 0.1 & 1000 & 0.4 \\
census
& 200 & 0.5 & 10000 & 2000 & 0.5 & 10000 & 0.4 \\
cervical & 2000 & 0.6 & 2000 & 200 & 0.5 & 2000 & 1 \\
credit & 4000 & 0.6 & 50000 & 4000 & 0.5 & 5000 & 0.005 \\
epileptic & 6000 & 0.5 & 100000 & 6000 & 0.5 & 80000 & 1 \\
isolet & 4000 & 0.6 & 100000 & 4000 & 0.5 & 500 & 1 \\
covtype & 6000 & 0.05 & 1000 & 6000 & 0.05 & 1000 & 0.03 \\
intrusion & 10000 & 0.03 & 2000 & 10000 & 0.03 & 2000 & 0.1 \\
\bottomrule
\end{tabular}}
\label{tab:tab_params}
\end{table*}

\subsection{Detailed results for binary class dataset}

In the main text we included the details for a multi-class dataset and here we also include the results across all the classification methods for a binary dataset in \tabref{credit_best} and \tabref{credit_avg}. 
We also include the best and average F1-score over five runs for the respective classification methods in \tabref{intrusion_best} and \tabref{intrusion_avg}. 
Notice that this average corresponds to the average reported in Table 1 in the main text.

\npdecimalsign{.}
\nprounddigits{2}
\begin{table*}[htb]
\caption{Performance comparison on Credit dataset. The highest performance in five runs.}
\centering
\scalebox{0.9}{
\begin{tabular}{l*{5}{|n{1}{2}n{1}{2}}}
\toprule
& \multicolumn{2}{ c| }{Real} & \multicolumn{2}{ c| }{DP-CGAN}  & \multicolumn{2}{ c| }{DP-MERF} & \multicolumn{2}{ c| }{DP-CGAN}  & \multicolumn{2}{ c }{\textbf{DP-MERF}} \\ 
& \multicolumn{2}{ c| }{} & \multicolumn{2}{ c| }{(non-priv)} &  \multicolumn{2}{ c| }{(non-priv)} & \multicolumn{2}{ c| }{($1,10^{-5}$)-DP} & \multicolumn{2}{ c }{($1,10^{-5}$)-DP} \\ 
\midrule
& ROC & PRC & ROC & PRC & ROC & PRC & ROC & PRC & ROC & PRC 
\\ 
Logistic Regression & 0.946 & 0.911 & 0.832 & 0.374 & 0.916 & 0.787 & 0.737 & 0.522 & 0.775 &  0.609 \\
Gaussian Naive Bayes & 0.9 &  0.801 & 0.85  &  0.394 & 0.92 &  0.757 & 0.8 & 0.546 & 0.653 & 0.484 \\
Bernoulli Naive Bayes & 0.886 & 0.838 & 0.576 & 0.192 & 0.893 & 0.817 & 0.665 & 0.424 & 0.897 & 0.742 \\
Linear SVM & 0.923 & 0.89 & 0.838 & 0.481 & 0.906 & 0.65 & 0.776 & 0.448 & 0.643 & 0.377 \\
Decision Tree & 0.911 & 0.822 & 0.742 & 0.321 & 0.916 & 0.692 & 0.576 & 0.222 & 0.719 & 0.582 \\
LDA & 0.873 & 0.819 & 0.856 & 0.528 & 0.815 & 0.677 & 0.58 & 0.241 & 0.694 & 0.506 \\
Adaboost & 0.935 & 0.889 & 0.833 & 0.506 & 0.926 & 0.849 & 0.615 & 0.322 & 0.75 & 0.628 \\
Bagging & 0.909 & 0.84 & 0.785 & 0.423 & 0.911 & 0.786 & 0.567 & 0.211 & 0.735 & 0.605 \\
Random Forest & 0.927 & 0.896 & 0.819 & 0.54 & 0.923 & 0.859 & 0.634 & 0.305 & 0.745 & 0.62 \\
GBM & 0.935 & 0.889 & 0.854 & 0.538 & 0.939 & 0.847 & 0.576 & 0.222 & 0.74 & 0.613 \\
Multi-layer perceptron & 0.923 & 0.89 & 0.828 & 0.467 & 0.91 & 0.74 & 0.779 & 0.552 & 0.659 & 0.436 \\
XGBoost & 0.944 & 0.913 & 0.805 & 0.487 & 0.936 & 0.866 & 0.703 & 0.532 & 0.724 & 0.59 \\
\midrule
Average & 0.914 & 0.863 & 0.802 & 0.438 & 0.909 & 0.777 & 0.667 & 0.379 & 0.728 & 0.566 \\
\bottomrule
\end{tabular}\label{tab:credit_best}} 
\end{table*}
\npnoround

\npdecimalsign{.}
\nprounddigits{3}
\begin{table*}[htb]
\caption{Performance comparison on Credit dataset. The average performance over five runs.}
\vskip 0.15in
\centering
\scalebox{1.1}{
\begin{tabular}{l*{2}{|n{1}{3}n{1}{3}}}
\toprule
& \multicolumn{2}{ c| }{DP-MERF} & \multicolumn{2}{ c }{DP-MERF} \\
& \multicolumn{2}{ c| }{(non-private)} & \multicolumn{2}{ c }{(private)} \\
& ROC & PRC & ROC
& PRC \\
\midrule
Logistic Regression & 0.9186 & 0.8082 & 0.796 & 0.665 \\
Gaussian Naive Bayes & 0.8976 & 0.7252 & 0.729 & 0.582 \\
Bernoulli Naive Bayes & 0.8794 & 0.7906 & 0.752 & 0.586 \\
Linear SVM & 0.8764 & 0.667 & 0.742 & 0.549 \\
Decision Tree & 0.9006 & 0.6996 & 0.775 & 0.650 \\
LDA & 0.8382 & 0.6968 & 0.725 & 0.544 \\
Adaboost & 0.912 & 0.8276 & 0.787 & 0.689 \\
Bagging & 0.9094 & 0.8046 & 0.811 & 0.709 \\
Random Forest & 0.9114 & 0.84 & 0.786 & 0.686 \\
GBM & 0.917 & 0.8124 & 0.807 & 0.707 \\
Multi-layer perceptron & 0.9054 & 0.7772 & 0.747 & 0.570 \\
XGBoost & 0.9154 & 0.8374 & 0.812 & 0.716 \\
\midrule
Average & 0.898 & 0.774 & 0.772 & 0.638 \\
\bottomrule
\end{tabular}\label{tab:credit_avg}}
\end{table*}
\npnoround


\begin{table*}[!t]
\caption{Performance comparison on Intrusion dataset. The highest performance in five runs.}
\vskip 0.15in
\centering
\scalebox{1.1}{
\begin{tabular}{l*{5}{|c}}
\toprule
& Real & DP-CGAN  & DP-MERF & DP-CGAN   & DP-MERF \\ 
&  & (non-priv) &  (non-priv) & ($1,10^{-5}$)-DP & ($1,10^{-5}$)-DP \\ 
\midrule
Logistic Regression & 0.948 & 0.710 & 0.926 & 0.567 & 0.940 \\
Gaussian Naive Bayes & 0.757 & 0.503 & 0.804 & 0.215 & 0.736 \\
Bernoulli Naive Bayes & 0.927 & 0.693 & 0.822 & 0.475 & 0.755 \\
Linear SVM & 0.983 & 0.639 & 0.922 & 0.915 & 0.937 \\
Decision Tree & 0.999 & 0.496 & 0.862 & 0.153 & 0.952 \\
LDA & 0.990 & 0.224 & 0.910 & 0.652 & 0.950 \\
Adaboost & 0.947 & 0.898 & 0.924 & 0.398 & 0.503 \\
Bagging & 1.000 & 0.499 & 0.914 & 0.519 & 0.956 \\
Random Forest & 1.000 & 0.497 & 0.941 & 0.676 & 0.943 \\
GBM & 0.999 & 0.501 & 0.924 & 0.255 & 0.933 \\
Multi-layer perceptron & 0.997 & 0.923 & 0.933 & 0.733 & 0.957 \\
XGBoost & 0.999 & 0.886 & 0.921 & 0.751 & 0.933 \\ \midrule
Average & 0.962 & 0.622 & 0.900 & 0.526 & 0.875 \\
\bottomrule
\end{tabular}
}\label{tab:intrusion_best}
\end{table*}

\clearpage

\begin{table*}[htb]
\caption{Performance comparison on Intrusion dataset. The average performance as F1 score over five runs.}
\centering
\scalebox{0.9}{
\begin{tabular}{l|c|c}
\toprule
& DP-MERF & DP-MERF \\
& (non-private) & (private) \\
\midrule
Logistic Regression & 0.891 & 0.928 \\
Gaussian Naive Bayes & 0.845 & 0.792 \\
Bernoulli Naive Bayes & 0.454 & 0.508 \\
Linear SVM & 0.890 & 0.917 \\
Decision Tree & 0.911 & 0.907 \\
LDA & 0.859 & 0.925 \\
Adaboost & 0.899 & 0.592 \\
Bagging & 0.926 & 0.922 \\
Random Forest & 0.904 & 0.923 \\
GBM & 0.901 & 0.926 \\
Multi-layer perceptron & 0.898 & 0.941 \\
XGBoost & 0.891 & 0.921 \\ 
\midrule
Average & 0.856 & 0.850 \\
\bottomrule
\end{tabular}\label{tab:intrusion_avg}}
\end{table*}


\section{Image data} \label{supp:mnist_data_details}

\subsection{Datasets}

Both digit and fashion MNIST datasets are loaded through the torchvision package and used without further preprocessing. 
Both datasets of size 60000 consist of samples from 10 classes, which are close to perfectly balanced. Each sample is a 28x28 pixel image and thus of significantly higher dimensionality than the tabular data we tested.

\subsection{Detailed results}

A detailed version of the results summarized in \tabref{summary_all_image} of the paper are shown below, for digit MNIST is \tabref{dmnist} and fashion MNIST in \tabref{fmnist}. All scores are the average of 5 independent runs of training a generator and evaluating the synthetic data it produced. The tables show that DP-MERF consistently outperforms the other approaches across models. The only exceptions are Gaussian Naive Bayes and XGBoost on MNIST, where GS-WGAN and DP-CGAN respectively perform slightly better.

\npdecimalsign{.}
\nprounddigits{2}
\begin{table*}[htb]
\caption{Test accuracy on digit MNIST data. Average over 5 runs (data generation \& model training). Best scores among private models are bold.}
\vskip 0.1in
\centering
\scalebox{0.95}{
\begin{tabular}{l|c|c|c|c|c|c|c}
\toprule
& Real & DP-CGAN & DP-GAN & GS-WGAN & DP-MERF & DP-MERF & DP-MERF  \\
& & $\epsilon=9.6$ & $\epsilon=9.6$ & $\epsilon=10$ & $\epsilon=\infty$ & $\epsilon=1$ & $\epsilon=0.2$ \\
\midrule                
Logistic Regression     & 0.930 & 0.600 & 0.702 & 0.741 & 0.772 & \textbf{0.769}  & 0.772  \\
Random Forest           & 0.969 & 0.638 & 0.538 & 0.460 & 0.714 & \textbf{0.685}  & 0.702  \\
Gaussian Naive Bayes    & 0.560 & 0.310 & 0.364 & \textbf{0.576} & 0.527 & 0.545  & 0.539  \\
Bernoulli Naive Bayes   & 0.840 & 0.610 & 0.702 & 0.699 & 0.746 &\textbf{ 0.750}  & 0.780  \\
Linear SVM              & 0.920 & 0.550 & 0.700 & 0.704 & 0.756 & \textbf{0.746}  & 0.726  \\
Decision Tree           & 0.880 & 0.340 & 0.255 & 0.326 & 0.443 & \textbf{0.456}  & 0.346  \\
LDA                     & 0.879 & 0.590 & 0.694 & 0.732 & 0.789 & \textbf{0.793}  & 0.753  \\
Adaboost                & 0.729 & 0.254 & 0.159 & 0.170 & 0.441 & \textbf{0.456}  & 0.362  \\
MLP                     & 0.978 & 0.564 & 0.652 & 0.744 & 0.807 & \textbf{0.807}  & 0.768  \\
Bagging                 & 0.928 & 0.430 & 0.282 & 0.387 & 0.624 & \textbf{0.602}  & 0.508  \\
GBM                     & 0.909 & 0.460 & 0.205 & 0.362 & 0.678 & \textbf{0.659}  & 0.552  \\
XGBoost                 & 0.912 & \textbf{0.614} & 0.459 & 0.408 & 0.525 & 0.555  & 0.509  \\
\midrule
Average                 & 0.870 & 0.500 & 0.476 & 0.526 & 0.652 & \textbf{0.652}  & 0.610  \\
\bottomrule
\end{tabular}
}\label{tab:dmnist}
\end{table*}
\npnoround

\npdecimalsign{.}
\nprounddigits{2}
\begin{table*}[htb]
\caption{Test accuracy on fashion MNIST data. Average over 5 runs (data generation \& model training). Best scores among private models are bold.}
\vskip 0.1in
\centering
\scalebox{0.95}{
\begin{tabular}{l|c|c|c|c|c|c|c}
\toprule
& Real & DP-CGAN & DP-GAN & GS-WGAN & DP-MERF  & DP-MERF & DP-MERF  \\
& & $\epsilon=9.6$ & $\epsilon=9.6$ & $\epsilon=10$ & $\epsilon=\infty$ & $\epsilon=1$ & $\epsilon=0.2$ \\
\midrule                
Logistic Regression     & 0.844 & 0.461 & 0.626 & 0.674 & 0.725 & \textbf{0.728} & 0.714  \\
Random Forest           & 0.875 & 0.482 & 0.573 & 0.498 & 0.657 & \textbf{0.684} & 0.553  \\
Gaussian Naive Bayes    & 0.585 & 0.286 & 0.149 & 0.505 & 0.598 & \textbf{0.575} & 0.467  \\
Bernoulli Naive Bayes   & 0.648 & 0.497 & 0.592 & 0.558 & 0.602 & \textbf{0.604} & 0.629  \\
Linear SVM              & 0.839 & 0.389 & 0.613 & 0.639 & 0.685 & \textbf{0.684} & 0.697  \\
Decision Tree           & 0.790 & 0.315 & 0.317 & 0.389 & 0.433 & \textbf{0.462} & 0.352  \\
LDA                     & 0.799 & 0.490 & 0.638 & 0.653 & 0.735 & \textbf{0.733} & 0.701  \\
Adaboost                & 0.561 & 0.217 & 0.224 & 0.275 & 0.291 & \textbf{0.359} & 0.258  \\
MLP                     & 0.879 & 0.459 & 0.601 & 0.647 & 0.739 & \textbf{0.738} & 0.696  \\
Bagging                 & 0.841 & 0.309 & 0.410 & 0.413 & 0.576 & \textbf{0.593} & 0.372  \\
GBM                     & 0.834 & 0.331 & 0.254 & 0.352 & 0.626 & \textbf{0.624} & 0.429  \\
XGBoost                 & 0.826 & 0.489 & 0.478 & 0.427 & 0.596 & \textbf{0.610} & 0.445  \\
\midrule
Average                 & 0.780 & 0.390 & 0.457 & 0.502 & 0.605 & \textbf{0.616} & 0.526  \\
\bottomrule
\end{tabular}
}\label{tab:fmnist}
\end{table*}
\npnoround

\section{Comparison with other methods}\label{supp:comparison_others}

\subsection{Comparison with \texorpdfstring{\cite{BalTolSch18}}{Balog et al.}.}

\vspace{-0.3cm}
\begin{wrapfigure}{r}{0.4\textwidth}
  \begin{center}
    \vspace{-13mm}
    \includegraphics[width=1\linewidth]{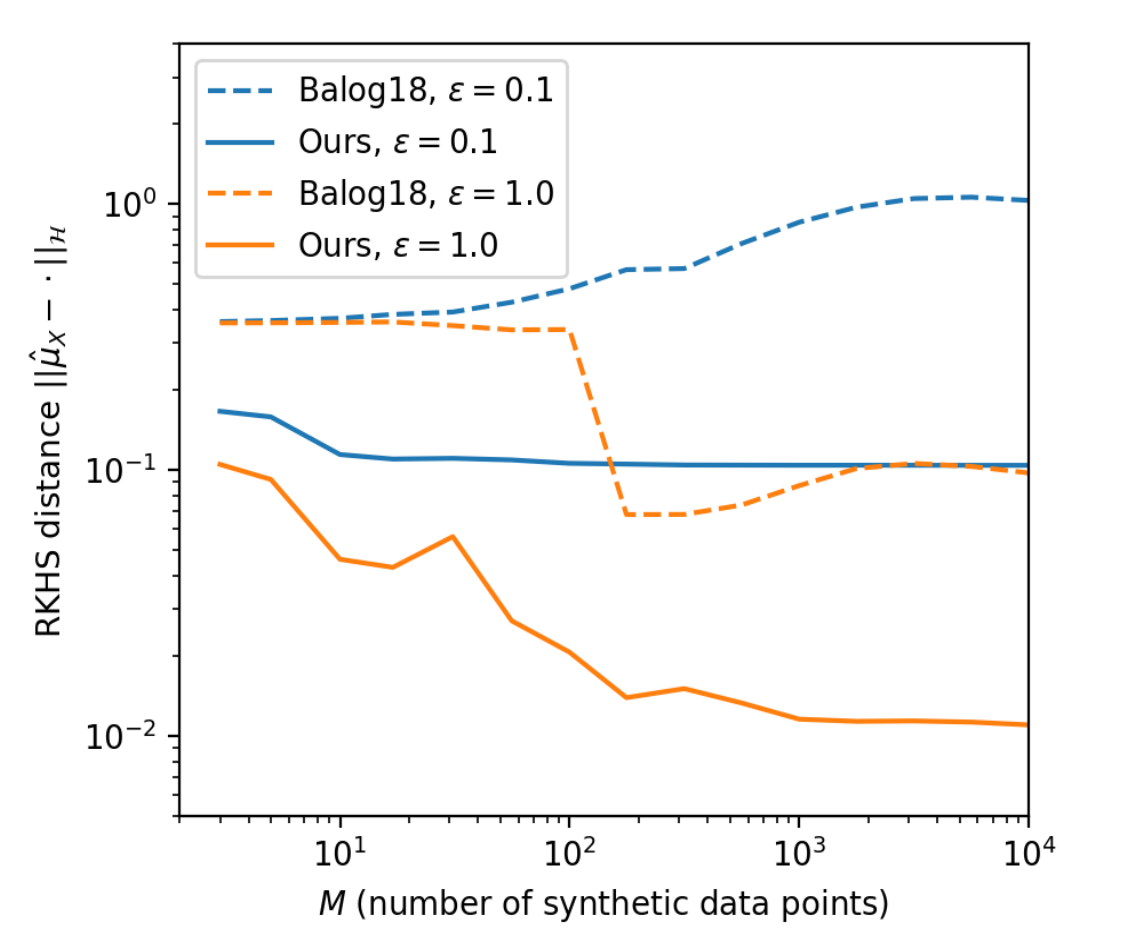}
  \end{center}
  \vspace{-4mm}
  \caption{Comparison to \cite{BalTolSch18}. }
  \label{fig:comparison_to_Balog18}
  \vspace{-4mm}
\end{wrapfigure}
Algorithm 2 in  \cite{BalTolSch18} uses the random features similar to ours, while it releases the privatized mean embedding in terms of a weighted sum of feature maps evaluated at synthetic datapoints. The challenge is that optimizing for the synthetic datapoints using the reduced-set method becomes harder in high dimensions.
To illustrate this point, we took the simulated data generated from  
\textit{5-dimensional} mixture of Gaussians (the dataset \cite{BalTolSch18} used). Unlike \cite{BalTolSch18}, our method directly trains a neural-net based generator, which can  effectively approximate the privatized kernel mean embedding of the data. As a result, our method reduces the distance (this metric \cite{BalTolSch18} used) between between the true kernel mean embedding $\hat{\mu}_x$ and that of the released dataset as we increase the number of synthetic datapoints, as shown in \figref{comparison_to_Balog18}.

\subsection{Comparison with \texorpdfstring{PrivBayes \cite{privbayes}}{PrivBayes}.}

We compare our method to PrivBayes
\cite{privbayes} using the published code from \cite{mckenna2019graphical}, which builds on the original code with \cite{zhang2018ektelo} as a wrapper. 
We test the model on the Adult and Census datasets used in our paper by creating a version $\mathcal{D}$ of the dataset where all continuous features are discretized, and a version $\mathcal{D}^*$ where the domain of all features is reduced to a max of 15 to reduce complexity. Following \cite{privbayes}, we measure $\alpha$-way marginals for varying levels of $\epsilon$-DP and compare them to DP-MERF at $(\epsilon, \delta)$-DP with $\delta=10^{-5}$. 
Optimizing the "usefulness" parameter $\theta$, we find, as in \cite{privbayes}, that $\theta=4$ is close to optimal in most settings. Results for the best $\theta$ are shown.
We observe that PrivBayes performs better at $\epsilon=1$, but is more affected by increased noise, so at $\epsilon=0.3$ the methods are roughly tied and at $\epsilon=0.1$ DP-MERF has lower error.

\vspace{-0.3cm}
\begin{table}[htb]
\centering
\scalebox{0.85}{
\begin{tabular}{lc|ccc|ccc||lc|ccc|ccc}
\toprule
\multicolumn{2}{ c| }{\multirow{2}{*}{Adult}} & \multicolumn{3}{ c| }{PrivBayes} & \multicolumn{3}{ c|| }{DP-MERF} & 
\multicolumn{2}{ c| }{\multirow{2}{*}{Census}} & \multicolumn{3}{ c| }{PrivBayes} & \multicolumn{3}{ c }{DP-MERF}\\ 
& & $\epsilon{=}1$ & $\epsilon {=} 0.3$ & $\epsilon{=}0.1$ & $\epsilon{=}1$ & $\epsilon{=}0.3$ & $\epsilon{=}0.1$
& & & $\epsilon{=}1$ & $\epsilon{=}0.3$ & $\epsilon{=}0.1$ & $\epsilon{=}1$ & $\epsilon{=}0.3$ & $\epsilon{=}0.1$
\\
\midrule
\multirow{2}{*}{$\mathcal{D}$} & $\alpha{=}3$ & 0.275 & 0.446 & 0.577 & 0.348 & 0.405 & 0.480
& \multirow{2}{*}{$\mathcal{D}$} & $\alpha{=}2$ & 0.131 & 0.180 & 0.291 & 

0.172 &
0.190 &
0.222
\\
& $\alpha{=}4$ & 0.377 & 0.547 & 0.673 & 0.468 & 0.508 & 0.590
& & $\alpha{=}3$ & 0.264 & 0.323 & 0.429 & 
%
0.291 &
0.302
& 0.337 \\
\midrule
\multirow{2}{*}{$\mathcal{D}^*$} & $\alpha{=}3$ & 0.182 & 0.284 & 0.317 & 0.235 & 0.287 & 0.352
& \multirow{2}{*}{$\mathcal{D}^*$} & $\alpha{=}2$ & 0.111 & 0.136 & 0.199 & 0.139 & 
0.140  
& 0.176 \\
& $\alpha{=}4$ & 0.257 & 0.371 & 0.401 & 0.301 & 0.363 & 0.453
& & $\alpha{=}3$ &  0.199 & 0.258 & 0.325 & 
%
0.228 & 0.234 & 0.269
\\
\bottomrule
\end{tabular}} 
\end{table}

%
It is important to stress that our approach is more general than PrivBayes in that \textit{(i)} it does not require discretization of the data and \textit{(ii)} scales to higher dimensionality and arbitrary domains. Bayesian network construction in PrivBayes for a $k$-degree graph with $d$ nodes (i.e. features) compares up to ${d \choose k}$ options on each iteration, which restricts $k$ to small values if $d$ is large. This means, e.g., testing PrivBayes on binarized MNIST ($d=784$) with any $k>2$ is infeasible.

\end{document}